\pdfoutput=1
\documentclass[lettersize,journal]{IEEEtran}

\usepackage[colorlinks,urlcolor=blue,linkcolor=blue,citecolor=blue]{hyperref}

\usepackage{color,array}

\usepackage{amsmath, amsfonts, amsthm, amssymb}
\usepackage{graphicx}
\usepackage{enumitem}

\usepackage{textcomp}
\usepackage{stfloats}
\usepackage{url}
\usepackage{verbatim}
\usepackage{balance}

\usepackage{graphics} %
\usepackage{epsfig} %
\usepackage{mathptmx} %
\usepackage{times} %
\usepackage{amsmath} %
\usepackage{amssymb}  %
\usepackage{siunitx}
\usepackage{hyperref}
\usepackage{multirow}
\usepackage{multicol}
\usepackage{nicefrac}
\usepackage{makecell}
\usepackage{longtable}
\usepackage{svg}
\usepackage{cite}
\usepackage{picinpar}
\usepackage[latin1]{inputenc}
\usepackage{colortbl}
\usepackage{soul}
\usepackage{multirow}
\usepackage{pifont}
\usepackage{color}
\usepackage{alltt}
\usepackage{siunitx}
\usepackage{epstopdf}
\usepackage{pbox}
\usepackage{amsfonts}
\usepackage{textcomp}
\usepackage{diagbox}
\usepackage{graphics} %
\usepackage{times} %
\usepackage{amssymb}  %
\usepackage{microtype}
\usepackage{subfigure}
\usepackage{booktabs} %
\usepackage{xfrac}
\usepackage{nicefrac}
\usepackage{multicol}
\usepackage{float}
\usepackage{array}
\usepackage{algorithm}
\usepackage{gensymb}
\usepackage{xr}
\usepackage{algpseudocode}
\usepackage{orcidlink}
\usepackage{caption}

\newtheorem{theorem}{Theorem}[section]
\newtheorem{proposition}{Proposition}[section]
\newtheorem{lemma}{Lemma}[section]

\DeclareMathOperator*{\argmax}{arg\,max}

\usepackage{eso-pic}
\usepackage{pdfpages}
\AddToShipoutPictureFG*{\AtPageLowerLeft{\put(\LenToUnit{0.75in},\LenToUnit{0.22in}){\parbox[b]{\dimexpr\paperwidth-1.5in\relax}{\centering\scriptsize
\copyright~2026 IEEE. Personal use of this material is permitted. Permission from IEEE must be obtained for all other uses, in any current or future media, including reprinting/republishing this material for advertising or promotional purposes, creating new collective works, for resale or redistribution to servers or lists, or reuse of any copyrighted component of this work in other works. Accepted for publication in IEEE Transactions on Affective Computing.}}}}

\begin{document}

\title{Evidential Multimodal Transformers with Dempster-Shafer Fusion for Emotion Recognition}

\author{
    Muhammad Usama \orcidlink{0000-0002-3834-2167} \IEEEauthorrefmark{1}\IEEEauthorrefmark{2} and Yeoreum Jung \orcidlink{0009-0008-7101-2726} \IEEEauthorrefmark{1}

    \thanks{
        \IEEEauthorrefmark{1}Braindeck Inc., Seoul, Republic of Korea (emails: usama@braindeck.net, summer@braindeck.net).
    }
    \thanks{
        \IEEEauthorrefmark{2}School of Electrical Engineering, Korea Advanced Institute of Science and Technology (KAIST), Daejeon 34141, Republic of Korea (email: usama@kaist.ac.kr).
    }
    \thanks{Corresponding author: Muhammad Usama (email: usama@kaist.ac.kr)}
    \thanks{This work was supported by Braindeck Inc. and by the National IT Industry Promotion Agency (NIPA), funded by the Ministry of Science and ICT (MSIT), Korea, under Grant No.PJT-25-040359.}
}

\markboth{IEEE Transactions on Affective Computing}%
{Usama \MakeLowercase{\textit{and}} Jung: Evidential Multimodal Transformers with Dempster-Shafer Fusion}

\maketitle

\begin{abstract}
Multimodal emotion recognition faces the challenge of overconfidence, where standard softmax-based methods assign high probability even when sensor data is corrupted or contradictory. For instance, facial expressions may suggest positive valence while vocal prosody indicates negative arousal, creating semantic conflict that standard averaging suppresses. We present an Evidential Multimodal Transformer with Dempster-Shafer Fusion (EMT-DS) that addresses this problem through formal uncertainty quantification. EMT-DS models predictions as Subjective Opinions (Dirichlet distributions) rather than point estimates, separating epistemic uncertainty (lack of evidence) from aleatoric uncertainty (sensor noise). Modality-specific transformer encoders generate belief mass distributions, which we fuse using Dempster's Rule of Combination. This mechanism mathematically distinguishes agreement from conflict by normalizing the joint belief mass with a conflict coefficient ($1/(1-c)$), so that contradictory evidence is explicitly quantified as conflict and the fused decision probability becomes diffuse rather than being forced into an unjustified consensus.

We demonstrate that EMT-DS achieves consistent performance gains across four benchmark datasets under subject-independent and speaker-independent evaluation. On IEMOCAP (audio+text), EMT-DS achieves 84.3\% accuracy versus 81.7\% for the best baseline DCT (+2.6 points, $p < 0.01$). On DEAP (EEG+physiological), under leave-one-subject-out evaluation we obtain 82.4\% accuracy versus 78.1\% for the strongest baseline (+4.3 points). On the highly imbalanced KEMDy20, EMT-DS reaches 96.8\% accuracy, and on the large-scale CMU-MOSEI benchmark it improves accuracy by 2.4 points over the strongest evidential baseline. Through controlled ablation, we demonstrate that EMT-DS reduces Expected Calibration Error from 0.153 (standard softmax transformer) to 0.042. Furthermore, the model maintains 71.2\% accuracy under severe noise ($\sigma = 1.0$ additive Gaussian corruption on audio features) where attention baselines drop to 45.3\%, and degrades gracefully when entire modalities are absent. These gains are achieved with $\mathcal{O}(MK^2)$ fusion complexity, significantly lower than the $\mathcal{O}(M^2d^2)$ of cross-modal attention. Finally, we demonstrate that the learned conflict metric effectively flags potential errors, enabling a downstream selective-review strategy that captures 80\% of misclassifications by reviewing just 20\% of samples. This capability to reliably reject unsafe predictions is critical for deployment in clinical and safety-critical systems.
\end{abstract}


\section{Introduction}
\label{sec:introduction}

Emotion recognition systems enable advanced human-computer interaction in mental health monitoring, intelligent tutoring, autonomous vehicles, and social robotics~\cite{picard2000affective,gunes2021emotion}. Human emotions manifest across multiple modalities: facial expressions convey valence, vocal prosody signals arousal, physiological signals reveal autonomic responses, and linguistic content provides semantic context~\cite{poria2017review,baltrusaitis2019multimodal}. Multimodal systems integrating these heterogeneous streams consistently outperform unimodal approaches~\cite{le2023multi,zhao2023tdfnet}.

\subsection{Problem Definition}
Despite advances in transformer-based multimodal fusion~\cite{vaswani2017attention,dhanith2024multimodal,liu2025tacfn}, existing methods face a fundamental problem: the inability to represent ignorance. We distinguish the task of emotion recognition (mapping sensors to labels) from the core problem of uncertainty quantification. Standard fusion methods, such as concatenation, attention~\cite{maji2023multimodal,waligora2024joint}, or mixture-of-experts~\cite{han2024fusemoe,mengara2025cag}, map inputs to a probability distribution on the simplex $\Delta^{K-1}$. The obstacle is the softmax normalization constraint $\sum \pi_k = 1$, which forces a deterministic distribution even when evidence is weak or conflicting. This ``forced confidence'' leads to brittle failure modes: (1) sensor failure where noise is interpreted as a confident prediction; (2) semantic conflict where contradictory modalities (e.g., sarcasm) are averaged out rather than flagged. Our solution is to reformulate the output as a Dirichlet distribution over the simplex, explicitly separating belief from uncertainty using Subjective Logic.

We present Evidential Multimodal Transformer with Dempster-Shafer Fusion (EMT-DS), integrating these frameworks. First, each modality is processed by specialized transformer encoders~\cite{dosovitskiy2021image,baevski2020wav2vec,devlin2019bert} followed by evidential layers predicting Dirichlet parameters. The loss combines Bayes risk with KL regularization that penalizes overconfidence on noisy inputs~\cite{sensoy2018evidential}. Second, a differentiable Dempster-Shafer (D-S) fusion layer computes belief masses, uncertainty, and conflict for each modality pair. Unlike attention mechanisms that learn arbitrary weights~\cite{maji2023multimodal}, our fusion follows mathematically grounded belief combination semantics, explicitly using conflict to detect complex emotional states or sensor failures. Third, we show that under low conflict, fused uncertainty decreases below individual modalities, while cross-modal disagreement is explicitly quantified as conflict and drives the fused decision probability toward a non-committal state, providing robustness. The KL regularizer also drives the evidence toward the vacuous opinion for out-of-distribution (OOD) samples, so that gradients vanish and noise is not fit.

We delineate the novelty of EMT-DS precisely. Evidential deep learning and Subjective Logic are established tools that we build upon rather than claim as contributions. The technical novelty lies in three elements. First, we design a fully differentiable Dempster-Shafer fusion layer that backpropagates gradients through the conflict coefficient, in contrast to prior Dempster-Shafer emotion-recognition methods~\cite{xu2020emotion,nemati2019hybrid,soroush2019emotion} that apply combination rules post hoc to the outputs of separately trained classifiers. Second, we treat the cross-modal conflict as an explicit supervised training signal and as a deployable abstention score for human review, rather than as a quantity to be normalized away. Third, we provide a theoretical account connecting the fusion to Bayes-optimal inference and connecting the evidential regularizer to out-of-distribution uncertainty, and we validate it under subject-independent, missing-modality, and backbone-agnostic protocols. Unlike attention or gating mechanisms that learn weights with no fixed semantics~\cite{maji2023multimodal,han2024fusemoe}, the belief and conflict quantities in EMT-DS carry a fixed probabilistic meaning and make cross-modal disagreement explicit through the conflict mass.

Validation on IEMOCAP~\cite{busso2008iemocap}, DEAP~\cite{koelstra2012deap}, KEMDy20~\cite{kemdataset2020}, and the large-scale CMU-MOSEI~\cite{zadeh2018multimodal} benchmark shows EMT-DS outperforms transformer-based~\cite{zhao2023tdfnet,le2023multi,tsai2019multimodal}, mixture-of-experts~\cite{mengara2025cag}, evidential~\cite{han2022trusted}, and cross-modal attention~\cite{maji2023multimodal,zhang2024cross} methods by 2.3 to 4.3 percentage points in accuracy and 2.4 to 4.2 points in macro F1-score on the subject-independent and speaker-independent benchmarks (IEMOCAP, DEAP, and KEMDy20), and remains ahead of the strongest recent 2026 baseline on the large-scale CMU-MOSEI corpus, while reducing Expected Calibration Error by up to a factor of three. Learned conflict patterns reveal systematic cross-modal dissonance in specific emotion pairs, providing interpretable insights.

Section~\ref{sec:related} reviews related work. Section~\ref{sec:methodology} presents the mathematical formulation. Section~\ref{sec:experiments} describes datasets, reports results, comparisons, and provides detailed ablation studies. Section~\ref{sec:discussion} analyzes key findings. Section~\ref{sec:conclusion} concludes with implications and future directions.

\section{Related Work}
\label{sec:related}

\subsection{Multimodal Emotion Recognition}
\label{subsec:multimodal_emotion}

Multimodal emotion recognition has evolved from early concatenation~\cite{poria2017review} to deep learning fusion. Early fusion concatenates low-level features~\cite{baltrusaitis2019multimodal}, enabling joint representations but requiring synchronization. Late fusion combines classifier outputs~\cite{rifkin2004}, providing robustness to missing modalities but discarding low-level interactions. Hybrid approaches~\cite{nemati2019hybrid} balance both strategies. Attention-based fusion~\cite{maji2023multimodal,waligora2024joint} dynamically weights modalities, but lacks a formal mechanism to reject disagreements.

Transformer architectures~\cite{vaswani2017attention} dominate recent work. Le et al.~\cite{le2023multi} proposed transformer fusion with emotion-level representation learning. Zhao et al.~\cite{zhao2023tdfnet} introduced TDFNet with deep-scale fusion across temporal resolutions. Maji et al.~\cite{maji2023multimodal} combined cross-modal transformers with temporal self-attention. Tsai et al.~\cite{tsai2019multimodal} introduced the Multimodal Transformer (MulT) with directional pairwise cross-modal attention for unaligned sequences, which remains a strong general-purpose baseline. Liu et al.~\cite{liu2025tacfn} proposed TACFN with adaptive cross-modal fusion, while Yi et al.~\cite{yi2025hyfuser} introduced dual cross-modal attention. Beyond attention, recent conversational emotion-recognition methods model richer interaction structure: Shi et al.~\cite{shi2025identity} drive multimodal fusion with speaker-identity and modality attributes through a speaker-centric graph, and Chen and Shi~\cite{chen2025dynamic} construct dynamic monologic and dialogic sub-hypergraphs to capture interaction-pattern-specific emotional relationships. These graph- and attribute-based designs improve accuracy by structuring cross-utterance interactions, but, like the attention and mixture-of-experts models, they produce softmax outputs without a formal account of uncertainty or cross-modal conflict; EMT-DS is complementary, contributing belief and conflict quantities with fixed probabilistic meaning that could augment such architectures. Mixture-of-experts approaches like FuseMoE~\cite{han2024fusemoe} and CAG-MoE~\cite{mengara2025cag} provide modality specialization with sparse activation. However, all these methods rely on softmax outputs lacking formal uncertainty quantification, making them brittle when modalities conflict or sensors fail, the gap that EMT-DS addresses through evidential Dirichlet modeling and conflict-aware Dempster-Shafer fusion.

Research in 2026 has continued to pursue robustness and reliability in multimodal fusion. Wei et al.~\cite{wei2026unbiased} debias dynamic fusion with a noise-aware uncertainty estimator, Huang et al.~\cite{huang2026complementarity} route spectral bands under complementarity supervision, and Zhang et al.~\cite{zhang2026trustworthy} combine dynamic weighting with Dempster-Shafer evidence for trustworthy dual-modal inference. Closest to our setting, Grzeczkowicz et al.~\cite{grzeczkowicz2026uncertainty} also pair Dempster-Shafer theory with Dirichlet evidence for emotion recognition, but as a training-free procedure applied post hoc to the logits of a frozen model. EMT-DS differs in that its Dempster-Shafer fusion is learned end to end, with gradients flowing through the conflict coefficient, with conflict serving as an explicit supervised signal and abstention score, and with the accompanying calibration and out-of-distribution analysis that a post-hoc wrapper does not provide.

\subsection{Physiological-Based Emotion Recognition}
\label{subsec:physiological_emotion}

Physiological signals (EEG, EDA, HRV) directly measure autonomic responses to emotional stimuli~\cite{jenke2014feature,boucsein2012electrodermal}. On the DEAP benchmark~\cite{koelstra2012deap}, Sartipi et al.~\cite{sartipi2023hybrid} proposed spatiotemporal attention with graph-based smoothing. Yin et al.~\cite{yin2024research} demonstrated complementarity between EEG and electrooculography. Zhang et al.~\cite{zhang2024cross} introduced cross-modal credibility modeling for EEG systems with learned reliability scores. These credibility measures are heuristic rather than grounded in formal uncertainty frameworks, unlike EMT-DS's Subjective Logic foundation.

\subsection{Evidential Deep Learning and Subjective Logic}
\label{subsec:evidential_learning}

Evidential deep learning~\cite{sensoy2018evidential} reformulates classification as predicting Dirichlet parameters, with strength directly quantifying epistemic uncertainty. Li et al.~\cite{li2024hyper} extended this to hierarchical settings, while Li et al.~\cite{li2022region} applied it to medical segmentation with improved noise robustness. Subjective Logic~\cite{josang2016subjective} provides formal algebra for belief reasoning through opinions $\omega = (\mathbf{b}, u, \mathbf{a})$ with belief masses, uncertainty, and base rates. Koksalmis and Kabak~\cite{xu2020expert} applied Dempster-Shafer combination to large-scale sensor fusion, showing that explicit conflict handling improves classification reliability.

Recent work explored evidential multimodal fusion in other domains. Han et al.~\cite{han2022trusted} proposed trusted multi-view classification, which characterizes each view by a Dirichlet distribution and integrates views through Dempster-Shafer combination; it is the closest evidential-fusion baseline to our setting, although it targets generic multi-view inputs and does not model temporal emotional dynamics or treat conflict as a supervised signal. Huang et al.~\cite{huang2025deep} proposed deep evidential fusion for medical segmentation. Mohey El-Din et al.~\cite{mohey2024adaptive} developed adaptive fusion using Dempster-Shafer theory for visual-textual object recognition. However, these approaches focus on static images rather than temporal emotion dynamics. EMT-DS uniquely integrates evidential transformers with differentiable Dempster-Shafer fusion for emotion recognition, enabling end-to-end learning where conflict signals complex emotional states rather than being discarded.

\subsection{Uncertainty Quantification in Deep Learning}
\label{subsec:uncertainty_quantification}

Bayesian neural networks~\cite{gal2016dropout} and ensemble methods~\cite{kendall2017uncertainties} enable uncertainty estimation but require multiple forward passes, increasing cost. Gal and Ghahramani~\cite{gal2016dropout} interpreted dropout as Bayesian approximation for Monte Carlo sampling. Kendall and Gal~\cite{kendall2017uncertainties} distinguished aleatoric and epistemic uncertainty for safety-critical applications. Evidential methods offer single-pass uncertainty through predictive distribution parameterization. Pandey and Yu~\cite{pandey2023learn} showed evidential networks accumulate evidence across training samples, improving generalization. EMT-DS leverages this for multimodal emotion recognition, where evidential encoding captures varying modality reliability while Dempster-Shafer fusion provides formal conflict resolution.

\subsection{Dempster-Shafer Theory in Emotion Recognition}
\label{subsec:dempster_shafer_emotion}

Dempster-Shafer Theory~\cite{dempster1968generalization,shafer1976mathematical} has limited application to emotion recognition. Xu et al.~\cite{xu2020emotion} derived belief assignments from classifier confidences for multimodal video emotion recognition. Nemati et al.~\cite{nemati2019hybrid} employed weighted belief combination in hybrid latent fusion. Soroush et al.~\cite{soroush2018novel,soroush2019emotion} applied modified Dempster-Shafer theory to EEG recognition with nonlinear features. These treat Dempster-Shafer as post-hoc fusion on pre-trained outputs.

EMT-DS differs in three concrete respects: (1) Dirichlet-parameterized evidential layers are embedded in the architecture for end-to-end belief learning; (2) differentiable Dempster-Shafer fusion backpropagates gradients through conflict computation; (3) conflict serves as an explicit signal for complex emotional states rather than being normalized away. This enables learning when and why modalities conflict, providing both improved accuracy (2.3 to 4.3 percentage points on the subject-independent and speaker-independent benchmarks) and interpretable uncertainty estimates.

\section{Evidential Multimodal Transformer with Dempster-Shafer Fusion}
\label{sec:methodology}

This section presents the mathematical formulation of EMT-DS. We begin by establishing notation and the problem formulation, then detail the evidential encoding mechanism, the Dempster-Shafer fusion layer, the training procedure, and conclude with theoretical analysis. Figure~\ref{fig:architecture} provides an overview of the complete framework, illustrating the flow from raw multimodal inputs through transformer encoders, evidential heads, and Dempster-Shafer fusion to the final emotion prediction.

\subsection{Notation and Problem Formulation}
\label{subsec:problem_formulation}

Consider a multimodal emotion recognition task with $K$ discrete emotion classes. Let $\mathcal{X} = \{\mathbf{x}^1, \dots, \mathbf{x}^M\}$ denote the set of inputs from $M$ heterogeneous modalities, where $\mathbf{x}^m$ represents the raw input for modality $m \in \{1, \ldots, M\}$. In our experiments, we consider $M \in \{2, 3\}$ modalities depending on the dataset: audio waveforms, video frames, transcribed text, and physiological signals (EEG, EDA, or IBI). The objective is to predict the emotion label $y \in \{1, \ldots, K\}$ from $\mathcal{X}$.

Conventional deep learning approaches learn a function $f_\theta: \mathcal{X} \rightarrow \Delta^{K-1}$ that maps inputs to the $(K-1)$-dimensional probability simplex. Let $\boldsymbol{\pi} = [\pi_1, \ldots, \pi_K]^\top$ denote the latent random variable representing the true class probability distribution. Standard softmax classifiers estimate a point mass $\hat{\boldsymbol{\pi}}$, but this forces a deterministic distribution even when the input evidence is weak or conflicting. This limitation becomes problematic in three scenarios: (1) sensor failure where modality $m$ produces out-of-distribution or corrupted data; (2) semantic conflict where modalities provide contradictory signals; and (3) incomplete data where certain modalities are missing.

We reformulate the problem using Subjective Logic~\cite{josang2016subjective}, which represents knowledge about a proposition through a Subjective Opinion. For each modality $m$ and emotion class $k$, we define a Subjective Opinion $\omega^m$ as a triplet:
\begin{equation*}
\omega^m = (\mathbf{b}^m, u^m, \mathbf{a}^m),
\end{equation*}
where $\mathbf{b}^m = [b_1^m, \ldots, b_K^m]^{\top}$ is the belief mass vector, with $b_k^m \in [0,1]$ representing the degree of belief that modality $m$ assigns to emotion class $k$; $u^m \in [0,1]$ is the epistemic uncertainty (also called vacuity), representing the lack of evidence due to insufficient or ambiguous data; and $\mathbf{a}^m = [a_1^m, \ldots, a_K^m]^{\top}$ is the base rate (prior distribution), typically set to the uniform distribution $a_k^m = 1/K$ in the absence of domain knowledge. These components satisfy the additivity constraint:
\begin{equation}
\sum_{k=1}^K b_k^m + u^m = 1.
\label{eq:additivity_constraint}
\end{equation}
The projected probability vector $\mathbf{p}^m = [p_1^m, \ldots, p_K^m]^\top$ represents the expectation of $\boldsymbol{\pi}$ under the Dirichlet distribution:
\begin{equation}
p_k^m = \mathbb{E}[\pi_k] = b_k^m + u^m \cdot a_k^m.
\label{eq:projected_probability}
\end{equation}

This formulation provides a critical advantage: when evidence is strong, $u^m \approx 0$ and $p_k^m \approx b_k^m$, yielding confident predictions. When evidence is weak, $u^m \approx 1$ and $p_k^m \approx a_k^m$, reverting to the prior. Importantly, the uncertainty $u^m$ is explicitly quantified rather than hidden in probability mass spread across classes.

\begin{figure*}[t]
\centering
\includegraphics[width=0.7\textwidth,trim={0 0 0 0},clip]{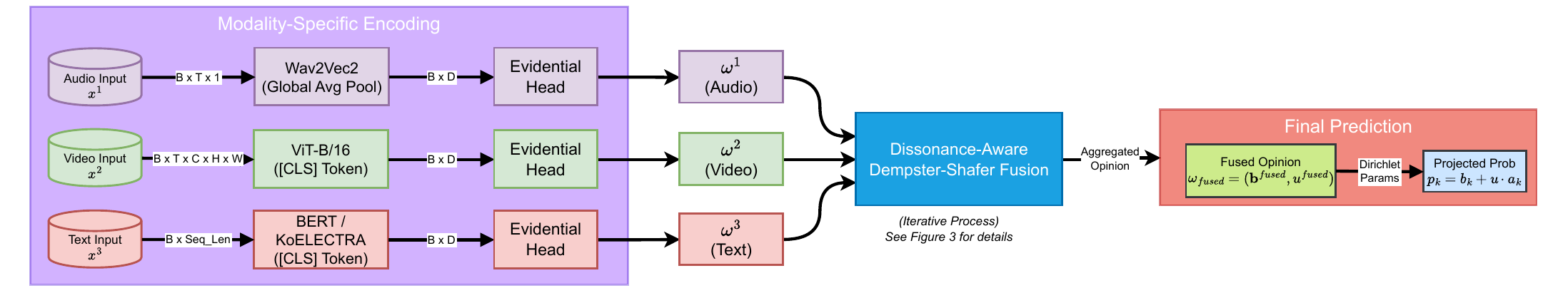}
\caption{Overall architecture of EMT-DS. Modality-specific encoders (e.g., Wav2Vec2, ViT-B/16, BERT/KoELECTRA, and a physiological 1D-CNN) extract features $h^m \in \mathbb{R}^d$ from heterogeneous inputs such as audio, video, text, and physiological signals. Evidential heads generate Subjective Opinions $\omega^m = (\mathbf{b}^m, u^m, \mathbf{a}^m)$ by predicting Dirichlet parameters via softplus-activated layers (Eq.~\ref{eq:evidence_layer}), yielding belief masses and uncertainty (Eq.~\ref{eq:belief_uncertainty}). The Dempster-Shafer fusion module iteratively combines opinions (Eqs.~\ref{eq:fused_belief},~\ref{eq:fused_uncertainty}), computing conflict $c$ to quantify cross-modal disagreement. The fused opinion produces final class probabilities via Eq.~\ref{eq:projected_probability}.}
\label{fig:architecture}
\end{figure*}

\subsection{Modality-Specific Evidential Encoding}
\label{subsec:evidential_encoding}

Each modality requires specialized processing to extract discriminative features. We employ transformer-based encoders tailored to the characteristics of each input modality, followed by an evidential head that generates Subjective Opinions. We employ transformer-based encoders to capture long-range temporal dependencies inherent in emotional expression (e.g., prosodic contours in audio or micro-expressions in video) which are often missed by local convolution operations. The complete evidential encoding pipeline is illustrated schematically in the Supplementary Material.

\subsubsection{Transformer-Based Feature Extraction}

For modality $m$, we apply a domain-specific transformer encoder $\mathcal{T}^m$ to extract a latent representation $\mathbf{h}^m \in \mathbb{R}^d$:
\begin{equation}
\mathbf{h}^m = \mathcal{T}^m(\mathbf{x}^m; \theta^m),
\label{eq:feature_extraction}
\end{equation}
where $\theta^m$ denotes the learnable parameters. For audio waveforms, we employ Wav2Vec2-Large~\cite{baevski2020wav2vec}, extracting the final hidden state with global average pooling to obtain $\mathbf{h}^{\text{audio}} \in \mathbb{R}^{1024}$. For facial expressions, we use Vision Transformer (ViT)~\cite{dosovitskiy2021image} with patch size 16, extracting the [CLS] token embedding to yield $\mathbf{h}^{\text{video}} \in \mathbb{R}^{768}$. For transcribed speech, we employ pretrained BERT~\cite{devlin2019bert} or KoELECTRA~\cite{park2020koelectra} for Korean text, using the [CLS] token representation to provide $\mathbf{h}^{\text{text}} \in \mathbb{R}^{768}$. For physiological signals (EEG, EDA, IBI), we use 1D convolutional networks followed by temporal self-attention to capture local patterns and long-range dependencies, producing $\mathbf{h}^{\text{physio}} \in \mathbb{R}^{512}$. All pretrained encoders are used as frozen feature extractors so that the reported gains are attributable to the evidential fusion mechanism rather than to encoder fine-tuning; the evidential heads and fusion layers are trained from scratch (Section~\ref{sec:experiments}).

\subsubsection{Evidential Output Layer}

Our primary technical contribution lies in replacing the standard softmax classifier with an evidential head that outputs Dirichlet distribution parameters. Following the evidential deep learning framework~\cite{sensoy2018evidential}, we model the class probability distribution $\boldsymbol{\pi}$ as a random variable drawn from a Dirichlet distribution:
\begin{equation}
\boldsymbol{\pi} \sim \text{Dir}(\boldsymbol{\alpha}^m) = \frac{1}{B(\boldsymbol{\alpha}^m)} \prod_{k=1}^K \pi_k^{\alpha_k^m - 1},
\label{eq:dirichlet_distribution}
\end{equation}
where $\boldsymbol{\alpha}^m = [\alpha_1^m, \ldots, \alpha_K^m]^{\top}$ are the concentration parameters and $B(\boldsymbol{\alpha}^m)$ is the multivariate beta function.

To ensure non-negativity of the Dirichlet parameters, we apply a softplus activation to a linear projection of the feature representation. Let $\epsilon = 10^{-6}$ denote the numerical stability constant (standard float32 machine epsilon). The raw evidence vector is computed as:
\begin{equation}
\tilde{\mathbf{e}}^m = \text{Softplus}(\mathbf{W}^m \mathbf{h}^m + \mathbf{b}^m) + \epsilon,
\label{eq:evidence_layer}
\end{equation}
where $\mathbf{W}^m \in \mathbb{R}^{K \times d}$ and $\mathbf{b}^m \in \mathbb{R}^K$ are learnable parameters, and $\text{Softplus}(x) = \log(1 + e^x)$ ensures positivity. To make the evidence magnitude bounded by construction rather than merely small in practice, we project the raw evidence onto the $L_1$ ball of radius $E_{\max}$:
\begin{equation}
\mathbf{e}^m = \tilde{\mathbf{e}}^m \cdot \min\!\left(1, \frac{E_{\max}}{\|\tilde{\mathbf{e}}^m\|_1}\right), \qquad E_{\max} = 1000.
\label{eq:evidence_cap}
\end{equation}
The projection is the identity whenever $\|\tilde{\mathbf{e}}^m\|_1 \le E_{\max}$, and otherwise rescales the evidence so that $\|\mathbf{e}^m\|_1 = E_{\max}$; in either case $\|\mathbf{e}^m\|_1 = \min(\|\tilde{\mathbf{e}}^m\|_1, E_{\max}) \le E_{\max}$ and $\mathbf{e}^m \ge \mathbf{0}$. Because $E_{\max}$ is set an order of magnitude above the empirically observed evidence norm ($\|\mathbf{e}^m\|_1 \approx 100$, Section~\ref{subsec:robustness}), the cap leaves the low-evidence regime untouched, and in particular preserves the vacuous limit $\mathbf{e}^m = \mathbf{0}$ ($u^m = 1$), while guaranteeing the finite evidence bound used in Theorem~\ref{thm:zadeh_resolution}. The vector $\mathbf{e}^m = [e_1^m, \ldots, e_K^m]^{\top}$ represents the bounded evidence for each class.

The Dirichlet concentration parameters are defined as:
\begin{equation}
\alpha_k^m = e_k^m + 1, \quad k = 1, \ldots, K.
\label{eq:alpha_definition}
\end{equation}
The additive constant 1 ensures $\alpha_k^m \geq 1$, corresponding to the non-informative Dirichlet prior when $e_k^m = 0$.

The mapping between Dirichlet parameters and Subjective Logic components is established through the Dirichlet strength:
\begin{equation}
S^m = \sum_{k=1}^K \alpha_k^m = \sum_{k=1}^K (e_k^m + 1) = \left(\sum_{k=1}^K e_k^m\right) + K.
\label{eq:dirichlet_strength}
\end{equation}
The belief mass and uncertainty are then computed as:
\begin{equation}
b_k^m = \frac{e_k^m}{S^m}, \quad u^m = \frac{K}{S^m}.
\label{eq:belief_uncertainty}
\end{equation}
Note that these definitions automatically satisfy the additivity constraint in Eq.~\ref{eq:additivity_constraint}.

Uncertainty Decomposition: This framework explicitly separates two types of uncertainty. Epistemic uncertainty (or vacuity) is captured by $u^m$, reflecting the lack of evidence (e.g., OOD data) where the model cannot form a distinct opinion. Aleatoric uncertainty (data noise or ambiguity) is captured by the entropy of the belief distribution $\mathbf{b}^m$. For example, a sample with $\mathbf{b}^m = [0.45, 0.45, 0.1, 0]$ has low epistemic uncertainty ($u^m=0$) but high aleatoric uncertainty due to class overlap, whereas a sample with $\mathbf{b}^m = [0, 0, 0, 0]$ and $u^m=1$ is purely epistemically uncertain.

To make the Subjective Logic quantities accessible to readers outside the Dempster-Shafer and Subjective Logic communities, Table~\ref{tab:intuition} summarizes their intuitive meaning and contrasts each with the familiar quantity it is most often confused with. The central distinction is that belief mass is not softmax confidence and uncertainty is not entropy: belief masses need not sum to one, and the residual mass $u^m$ explicitly measures missing evidence rather than the spread of a fixed distribution. As a concrete illustration with $K=4$ classes, an audio modality producing evidence $\mathbf{e}^{\text{audio}} = [8, 0, 0, 0]$ has Dirichlet strength $S = 8 + 4 = 12$, belief $b_1 = 8/12 = 0.667$, and uncertainty $u = 4/12 = 0.333$, giving a projected probability $p_1 = 0.667 + 0.333 \times 0.25 = 0.75$. A second modality producing no evidence, $\mathbf{e} = [0,0,0,0]$, yields $\mathbf{b} = \mathbf{0}$ and $u = 1$, expressing total ignorance. A standard softmax classifier cannot represent this state: a uniform softmax output $[0.25, 0.25, 0.25, 0.25]$ and the vacuous opinion ($u=1$) yield identical point probabilities yet carry opposite epistemic meaning, a distinction EMT-DS preserves through $u$.

\begin{table}[t]
\centering
\caption{Intuitive interpretation of the Subjective Logic quantities used in EMT-DS, and the familiar quantities they should not be confused with.}
\label{tab:intuition}
\resizebox{\columnwidth}{!}{%
\begin{tabular}{lll}
\toprule
\textbf{Quantity} & \textbf{Intuitive meaning} & \textbf{Not the same as} \\
\midrule
Belief mass $b_k^m$ & Evidence-backed support for class $k$ & Softmax confidence \\
Uncertainty $u^m$ & Amount of missing evidence (vacuity) & Predictive entropy \\
Base rate $a_k^m$ & Prior used when evidence is absent & A learned gating weight \\
Conflict $c$ & Extent of cross-modal disagreement & A reliability or similarity score \\
Projected prob. $p_k^m$ & Decision-time probability $b_k^m + u^m a_k^m$ & Raw softmax output \\
\bottomrule
\end{tabular}
}
\end{table}

\subsection{Dempster-Shafer Fusion}
\label{subsec:dempster_shafer_fusion}

Given Subjective Opinions {$\omega^1, \ldots, \omega^M$} from $M$ modalities, we must fuse them into a joint opinion $\omega^{\text{fused}}$. Traditional approaches such as weighted averaging or concatenation do not provide a mechanism for handling contradictory evidence. We employ Dempster's Rule of Combination~\cite{dempster1968generalization,shafer1976mathematical}, which explicitly computes the conflict between opinions and adjusts the fused belief accordingly. We favor Dempster-Shafer over Yager or simple Bayesian fusion because it strictly penalizes conflicting evidence through the normalization term, providing a mathematically grounded rejection of inconsistent multimodal inputs (e.g., in cases of sensor failure or sarcasm) rather than averaging them out. This contrasts with Yager's rule, which redistributes conflict to the universal set (uncertainty) without re-weighting the beliefs, or Bayesian averaging which ignores conflict entirely. The computational flow of the fusion module is illustrated schematically in the Supplementary Material.

\subsubsection{Pairwise Belief Fusion}

Consider two modalities $i$ and $j$ with Subjective Opinions $\omega^i = (\mathbf{b}^i, u^i, \mathbf{a}^i)$ and $\omega^j = (\mathbf{b}^j, u^j, \mathbf{a}^j)$. The fusion operation $\omega^i \oplus \omega^j$ proceeds as follows:

The conflict (or dissonance) measures the degree to which the two opinions assign belief to mutually exclusive hypotheses:
\begin{equation*}
c_{ij} = \sum_{k=1}^K \sum_{\ell \neq k} b_k^i b_\ell^j.
\end{equation*}
This simplifies to:
\begin{equation}
c_{ij} = \left(\sum_{k=1}^K b_k^i\right) \left(\sum_{k=1}^K b_k^j\right) - \sum_{k=1}^K b_k^i b_k^j = (1 - u^i)(1 - u^j) - \sum_{k=1}^K b_k^i b_k^j.
\label{eq:conflict_simplified}
\end{equation}
Intuitively, $c_{ij} = 0$ when the modalities agree perfectly (e.g., both strongly support ``happy''), while $c_{ij} \rightarrow 1$ when they completely disagree (e.g., one supports ``happy'' and the other ``angry'').

The fusion rule renormalizes the belief mass by discarding the conflicting mass. Let $n_{ij} \in [1, \infty)$ denote the conflict-based normalization factor:
\begin{equation}
n_{ij} = \frac{1}{1 - c_{ij}}.
\label{eq:normalization_factor}
\end{equation}
When conflict is low, $n_{ij} \approx 1$. When conflict is high, $n_{ij}$ becomes large, redistributing mass to account for the inconsistency.

The fused belief mass for class $k$ is:
\begin{equation}
b_k^{i \oplus j} = n_{ij} \left( b_k^i b_k^j + b_k^i u^j + b_k^j u^i \right).
\label{eq:fused_belief}
\end{equation}
The three terms represent: (1) both modalities agreeing on class $k$ ($b_k^i b_k^j$), (2) modality $i$ supports $k$ while $j$ is uncertain ($b_k^i u^j$), and (3) modality $j$ supports $k$ while $i$ is uncertain ($b_k^j u^i$).

The fused uncertainty is:
\begin{equation}
u^{i \oplus j} = n_{ij} \left( u^i u^j \right).
\label{eq:fused_uncertainty}
\end{equation}
This product structure implies that the combined uncertainty decreases exponentially as we add informative modalities (those with low $u$).

\subsubsection{Multi-Modal Fusion Algorithm}

For $M > 2$ modalities, we iteratively apply pairwise fusion. Algorithm~\ref{alg:dempster_fusion} summarizes the forward pass of EMT-DS.

\begin{algorithm}[t]
\caption{EMT-DS Forward Pass}
\label{alg:dempster_fusion}
\begin{algorithmic}[1]
\Require Multimodal inputs $\mathcal{X} = \{x^1, \ldots, x^M\}$, number of classes $K$
\Ensure Fused belief $\mathbf{b}^{\text{fused}}$, fused uncertainty $u^{\text{fused}}$, conflict $c^{\text{total}}$
\State // Step 1: Modality-Specific Evidence Generation
\For{$m = 1$ to $M$}
    \State $\mathbf{h}^m \gets \mathcal{T}^m(\mathbf{x}^m; \theta^m)$ \Comment{Extract features via Eq.~\ref{eq:feature_extraction}}
    \State $\tilde{\mathbf{e}}^m \gets \text{Softplus}(\mathbf{W}^m \mathbf{h}^m + \mathbf{b}^m) + \epsilon$ \Comment{Raw evidence via Eq.~\ref{eq:evidence_layer}}
    \State $\mathbf{e}^m \gets \tilde{\mathbf{e}}^m \cdot \min(1, E_{\max}/\|\tilde{\mathbf{e}}^m\|_1)$ \Comment{Evidence cap via Eq.~\ref{eq:evidence_cap}}
    \State $S^m \gets \sum_{k=1}^K (e_k^m + 1)$ \Comment{Dirichlet strength via Eq.~\ref{eq:dirichlet_strength}}
    \State $b_k^m \gets e_k^m / S^m$, $u^m \gets K / S^m$ for $k = 1, \ldots, K$ \Comment{Beliefs and uncertainty via Eq.~\ref{eq:belief_uncertainty}}
\EndFor
\State
\State // Step 2: Multi-source conflict (permutation-invariant, Eq.~\ref{eq:total_conflict})
\State $c^{\text{total}} \gets \dfrac{2}{M(M-1)} \sum_{1 \le i < j \le M} \big[(1 - u^i)(1 - u^j) - \sum_{k=1}^K b_k^i b_k^j\big]$
\State
\State // Step 3: Iterative Dempster-Shafer belief fusion
\State Initialize $\mathbf{b}^{\text{agg}} \gets \mathbf{b}^1$, $u^{\text{agg}} \gets u^1$
\For{$m = 2$ to $M$}
    \State $c \gets (1 - u^{\text{agg}})(1 - u^m) - \sum_{k=1}^K b_k^{\text{agg}} b_k^m$ \Comment{Pairwise conflict, Eq.~\ref{eq:conflict_simplified}}
    \State $n \gets 1 / (1 - c)$ \Comment{Normalization via Eq.~\ref{eq:normalization_factor}}
    \For{$k = 1$ to $K$}
        \State $b_k^{\text{temp}} \gets n \cdot (b_k^{\text{agg}} b_k^m + b_k^{\text{agg}} u^m + b_k^m u^{\text{agg}})$ \Comment{Via Eq.~\ref{eq:fused_belief}}
    \EndFor
    \State $u^{\text{temp}} \gets n \cdot (u^{\text{agg}} u^m)$ \Comment{Via Eq.~\ref{eq:fused_uncertainty}}
    \State $\mathbf{b}^{\text{agg}} \gets \mathbf{b}^{\text{temp}}$, $u^{\text{agg}} \gets u^{\text{temp}}$
\EndFor
\State $\mathbf{b}^{\text{fused}} \gets \mathbf{b}^{\text{agg}}$, $u^{\text{fused}} \gets u^{\text{agg}}$
\State \Return $\mathbf{b}^{\text{fused}}$, $u^{\text{fused}}$, $c^{\text{total}}$
\end{algorithmic}
\end{algorithm}

Implementation Note: To ensure numerical stability when $c_{ij}$ approaches 1 (near-total conflict), we apply a small threshold: if $c_{ij} > 0.999$, we clip the conflict. Specifically, the normalization term is capped at $n_{ij} = 1/(1 - 0.999) = 1000$ to prevent floating-point overflow while preserving the gradient signal that high conflict reduces the effective weight of the fusion.

For more than two modalities we summarize cross-modal disagreement by the mean pairwise conflict computed from the original per-modality opinions,
\begin{equation}
c^{\text{total}} = \frac{2}{M(M-1)} \sum_{1 \le i < j \le M} c_{ij}, \qquad c_{ij} = (1-u^i)(1-u^j) - \sum_{k=1}^K b_k^i b_k^j.
\label{eq:total_conflict}
\end{equation}
Each pairwise term satisfies $c_{ij} = \sum_{k \ne \ell} b_k^i b_\ell^j \in [0,1)$, so $c^{\text{total}} \in [0,1)$, and because the sum ranges over unordered pairs of the fixed per-modality opinions, $c^{\text{total}}$ is permutation invariant with respect to the fusion order. For $M = 2$ it reduces to the single pairwise conflict $c_{12}$, so the two-modality analyses in this paper are unaffected by the choice of definition. Keeping the measure in $[0,1)$ for any number of modalities makes the threshold $\tau = 0.5$ used below a valid midpoint of the conflict range.

\subsection{Training Procedure and Loss Functions}
\label{subsec:training_procedure}

Training evidential networks requires a specialized loss function that respects the Dirichlet parameterization. We cannot use standard cross-entropy loss, as the output is not a fixed probability vector but a distribution over probability vectors.

\subsubsection{Bayes Risk Loss}

For a training sample with one-hot ground truth label $\mathbf{y} = [y_1, \ldots, y_K]^{\top}$ (where $y_k = 1$ if $k$ is the true class, else 0), we minimize the expected Bayes risk under the Dirichlet predictive distribution. The expected probability for class $k$ under $\text{Dir}(\boldsymbol{\alpha}^m)$ is:
\begin{equation*}
\mathbb{E}_{\mathbf{p} \sim \text{Dir}(\boldsymbol{\alpha}^m)}[p_k] = \frac{\alpha_k^m}{S^m} = \frac{e_k^m + 1}{S^m}.
\end{equation*}
The Bayes risk loss for modality $m$ is defined as the sum of squared errors (SSE) between the ground truth and the expected probability, plus the variance term that penalizes high uncertainty:
\begin{equation}
\mathcal{L}_{\text{risk}}^m = \sum_{k=1}^K \left( y_k - \mathbb{E}[p_k] \right)^2 + \sum_{k=1}^K \text{Var}(p_k),
\label{eq:bayes_risk}
\end{equation}
where the variance of the Dirichlet distribution is:
\begin{equation*}
\text{Var}(p_k) = \frac{\alpha_k^m (S^m - \alpha_k^m)}{(S^m)^2 (S^m + 1)}.
\end{equation*}

\subsubsection{KL Divergence Regularization}

A critical component is the KL divergence regularizer that penalizes deviations from the uniform Dirichlet prior $\text{Dir}(\mathbf{1})$ (where $\mathbf{1} = [1, \ldots, 1]^{\top}$). This encourages the model to predict maximum uncertainty when the input lacks discriminative features:
\begin{equation*}
\mathcal{L}_{\text{KL}}^m = \lambda_t \cdot \text{KL}\left[ \text{Dir}(\boldsymbol{\alpha}^m) \parallel \text{Dir}(\mathbf{1}) \right],
\end{equation*}
where $\lambda_t$ is a time-dependent annealing coefficient that increases linearly from 0.01 to 1.0 over the first 10 epochs. The KL divergence between two Dirichlet distributions is:
\begin{equation*}
\begin{aligned}
\text{KL}\left[ \text{Dir}(\boldsymbol{\alpha}^m) \parallel \text{Dir}(\mathbf{1}) \right] &= \log \frac{B(\mathbf{1})}{B(\boldsymbol{\alpha}^m)} + \sum_{k=1}^K (\alpha_k^m - 1) \left[ \psi(\alpha_k^m) - \psi(S^m) \right],
\end{aligned}
\end{equation*}
where $\psi(\cdot)$ is the digamma function and $B(\boldsymbol{\alpha}^m) = \frac{\prod_{k=1}^K \Gamma(\alpha_k^m)}{\Gamma(S^m)}$ is the multivariate beta function.

\subsubsection{Total Loss}

The total loss function aggregates the risk and KL terms across all modalities and the fused opinion:
\begin{equation}
\mathcal{L}_{\text{total}} = \sum_{m=1}^M \left( \mathcal{L}_{\text{risk}}^m + \mathcal{L}_{\text{KL}}^m \right) + \mathcal{L}_{\text{risk}}^{\text{fused}} + \mathcal{L}_{\text{KL}}^{\text{fused}} + \lambda_c \mathcal{L}_{\text{conflict}},
\label{eq:total_loss}
\end{equation}
where $\mathcal{L}_{\text{conflict}}$ is an auxiliary loss that encourages the model to learn meaningful conflict patterns with weight $\lambda_c$ (conflict loss coefficient). Let $\tau = 0.5$ denote the conflict threshold, chosen as the midpoint of the conflict range $[0, 1)$ (Eq.~\ref{eq:total_conflict}), representing the crossover between conflicting and supportive evidence. The conflict loss is defined as:
\begin{equation}
\mathcal{L}_{\text{conflict}} = \max(0, \tau - c^{\text{total}}) \cdot \mathbb{I}[\text{misclassified}],
\label{eq:conflict_loss}
\end{equation}
where $\mathbb{I}[\text{misclassified}]$ is the indicator function for incorrect predictions. By penalizing $c < \tau$ only when the model is incorrect, we encourage the network to be conflicted (uncertain) rather than confidently wrong. The indicator is computed from an $\argmax$ comparison and is therefore non-differentiable; we apply it as a detached stop-gradient mask, so that gradients flow only through the conflict magnitude $c^{\text{total}}$ and never through the $\argmax$ operation. Because the mask only gates which samples contribute to $\mathcal{L}_{\text{conflict}}$, it introduces no instability in early training, when most samples are misclassified and the term simply behaves as a uniform penalty on low conflict; as accuracy improves, the mask sparsifies and the penalty concentrates on the remaining errors.

The fused evidential terms $\mathcal{L}_{\text{risk}}^{\text{fused}}$ and $\mathcal{L}_{\text{KL}}^{\text{fused}}$ are computed by mapping the fused opinion $(\mathbf{b}^{\text{fused}}, u^{\text{fused}})$ back to Dirichlet parameters. Inverting Eq.~\ref{eq:belief_uncertainty}, the fused Dirichlet strength is $S^{\text{fused}} = K / u^{\text{fused}}$, the fused evidence is $e_k^{\text{fused}} = b_k^{\text{fused}} \, S^{\text{fused}}$, and the fused concentration parameters are $\alpha_k^{\text{fused}} = e_k^{\text{fused}} + 1$. Substituting these parameters into Eq.~\ref{eq:bayes_risk} and the KL term supervises the fused opinion with the same evidential objective as the individual modalities, closing the differentiable path from raw inputs through fusion to the loss.

We optimize the model parameters $\Theta = \{\theta^1, \ldots, \theta^M, \mathbf{W}^1, \ldots, \mathbf{W}^M\}$ using the AdamW optimizer~\cite{loshchilov2019adamw} with learning rate $\eta = 2 \times 10^{-4}$, weight decay $\lambda_w = 0.01$, and cosine annealing schedule over 50 epochs.

\subsection{Theoretical Analysis}
\label{subsec:theoretical_analysis}

We establish three theoretical results for EMT-DS, stating explicitly the assumptions each relies on and separating them from the empirical observations they explain. First, we show that the evidential parameterization removes the fusion divergence associated with Zadeh's Paradox, with the required evidence bound guaranteed by construction through an architectural cap (Eq.~\ref{eq:evidence_cap}). Second, we show that evidential Dempster-Shafer fusion is Bayes-consistent under conditional independence, which explains our empirical Expected Calibration Error reduction under weak cross-modal error correlation. Third, we show that the KL regularizer is uniquely minimized by the vacuous opinion, encoding a preference for maximum uncertainty in the absence of evidence, and we corroborate the resulting out-of-distribution behavior empirically.

\subsubsection{Unconditional Stability via an Architectural Evidence Cap}

A fundamental critique of classical Dempster-Shafer theory is Zadeh's Paradox~\cite{zadeh1986simple}, which demonstrates that when conflict $c \to 1$, the normalization term $n = 1/(1-c) \to \infty$. This causes the fusion operator to become discontinuous and hyper-sensitive to noise, severely limiting its use in deep learning. We show that the Dirichlet-parameterized evidential mapping removes this singularity. The required evidence bound is not left to chance: the $L_1$ projection in Eq.~\ref{eq:evidence_cap} guarantees $\|\mathbf{e}^m\|_1 \le E_{\max}$ for every input by construction, so the stability below is unconditional rather than an empirical regularity. The softplus activation and KL regularization additionally keep evidence magnitudes well within this cap in practice (measured $\|\mathbf{e}^m\|_1 \approx 100 \ll E_{\max}$, Section~\ref{subsec:robustness}), so the projection is inactive on observed data while the bound it provides remains permanently valid.

\begin{theorem}[Unconditional Stability of Evidential Dempster-Shafer Fusion]
\label{thm:zadeh_resolution}
Let $\mathcal{E}^M = \{\mathbf{e}^1, \dots, \mathbf{e}^M\}$ be the evidence vectors produced by the capped evidential layer (Eq.~\ref{eq:evidence_cap}), which satisfy $\|\mathbf{e}^m\|_1 \le E_{\max} < \infty$ by construction. Then the Evidential Dempster-Shafer fusion operator $\Phi: \mathcal{E}^M \to \Delta^{K-1}$ is topologically stable. Specifically, the conflict $c$ is strictly bounded away from $1$, satisfying:
\begin{equation}
    c \le 1 - \frac{K^M}{(E_{\max} + K)^M}
    \label{eq:conflict_bound}
\end{equation}
Consequently, the fusion mapping is globally $L$-Lipschitz continuous with $L = (E_{\max} + K)^M / K^M$, rendering it deterministically robust to adversarial input perturbations $\delta$.
\end{theorem}

\begin{proof}[Proof sketch]
The architectural cap (Eq.~\ref{eq:evidence_cap}) guarantees $\|\mathbf{e}^m\|_1 \le E_{\max}$, hence $u^m = K/S^m \ge K/(E_{\max}+K) > 0$. The non-conflicting mass then satisfies $1 - c_{ij} \ge u^i u^j \ge K^2/(E_{\max}+K)^2 > 0$, so the normalizer $n_{ij} = 1/(1-c_{ij})$ cannot diverge; induction over the $M$ modalities yields Eq.~\ref{eq:conflict_bound} and the finite Lipschitz constant $L = (E_{\max}+K)^M/K^M$. The complete proof is given in the Supplementary Material.
\end{proof}

Theorem~\ref{thm:zadeh_resolution} shows that, conditioned on bounded evidence, forcing uncertainty $u$ to be inversely proportional to the Dirichlet strength through Eq.~\ref{eq:belief_uncertainty} prevents the normalization factor from diverging. Conflicting modalities are then accommodated by the bounded normalizer rather than causing numerical instability.

We distinguish the guaranteed bound from the operating range. The evidence bound holds for every input by the architectural cap (Eq.~\ref{eq:evidence_cap}), so the stability of Theorem~\ref{thm:zadeh_resolution} is unconditional and the Lipschitz constant $L = (E_{\max} + K)^M / K^M$ is a genuine finite constant rather than an empirical regularity. This worst-case constant is large but never approached in practice: with $E_{\max} = 1000$, $M=4$, and $K=7$ it evaluates to $L \approx 4 \times 10^8$, whereas the cap radius is set an order of magnitude above the empirically measured evidence norm ($\|\mathbf{e}^m\|_1 \approx 100$, Section~\ref{subsec:robustness}), so the projection is the identity on all observed samples and the empirical Lipschitz behavior remains many orders of magnitude smaller. The cap plays the same role in evidence space that the conflict clip (Section~\ref{subsec:dempster_shafer_fusion}) plays in fusion space: it removes a worst-case singularity by construction while remaining inactive in the operating regime. Gradient clipping at threshold 1.0 is retained as an additional safeguard.

\subsubsection{Bayesian Consistency of Evidential Fusion}

We connect Dempster-Shafer fusion to Bayes-optimal inference, which explains why conflict-aware evidential fusion improves calibration. We first characterize the fused decision probability exactly, then show that under standard independence assumptions it recovers the Bayes posterior.

\begin{proposition}[Exact form of the fused decision probability]
\label{prop:fused_exact}
For two modalities with uniform base rate $a_k = 1/K$, the fused projected probability is exactly
\begin{equation}
p_k^{1\oplus 2} = n\Big[\, p_k^1 p_k^2 + \big(1-\tfrac{1}{K}\big)\big(u^2 p_k^1 + u^1 p_k^2\big) - \tfrac{K-1}{K^2}\,u^1 u^2 \,\Big], \quad n=\tfrac{1}{1-c}.
\label{eq:fused_exact}
\end{equation}
In the confident limit $u^1,u^2\to 0$ this reduces to the normalized product $p_k^{1\oplus 2} = p_k^1 p_k^2 / \sum_{\ell} p_\ell^1 p_\ell^2$, and when one modality is vacuous ($u^m\to 1$) it returns the opinion of the other modality.
\end{proposition}

\begin{theorem}[Bayesian consistency and calibration of Dempster-Shafer fusion]
\label{thm:calibration_bound}
Suppose the two modalities are conditionally independent given the label $Y$ and the class prior is uniform. If each modality is calibrated in the sense that $p_k^m = \mathbb{P}(Y=k\mid \mathbf{x}^m)$, then in the confident limit the fused prediction of Eq.~\ref{eq:fused_exact} equals the Bayes-optimal posterior,
\begin{equation}
p_k^{1\oplus 2} = \mathbb{P}(Y=k\mid \mathbf{x}^1,\mathbf{x}^2),
\label{eq:bayes_consistency}
\end{equation}
and is therefore itself calibrated. The Dempster-Shafer normalizer $n=1/(1-c)$ coincides with the Bayesian evidence term, and the conflict $c = 1 - \sum_\ell p_\ell^1 p_\ell^2$ is one minus the probability that the two modalities agree. For $u^m>0$ the fused prediction interpolates, by Proposition~\ref{prop:fused_exact}, between this Bayesian product and deferral to the more certain modality.
\end{theorem}

\begin{proof}[Proof sketch]
Conditional independence and a uniform prior give $\mathbb{P}(Y=k\mid \mathbf{x}^1,\mathbf{x}^2) \propto \mathbb{P}(\mathbf{x}^1\mid Y=k)\,\mathbb{P}(\mathbf{x}^2\mid Y=k)$. Bayes' rule per modality gives $\mathbb{P}(\mathbf{x}^m\mid Y=k)\propto p_k^m/\pi_k$, so with uniform $\pi_k$ the joint posterior is proportional to $p_k^1 p_k^2$. Normalizing and invoking the confident limit of Proposition~\ref{prop:fused_exact} yields Eq.~\ref{eq:bayes_consistency}; the fused prediction then equals the true conditional and is calibrated by definition. The complete argument, including the $u^m>0$ interpolation, is given in the Supplementary Material.
\end{proof}

Theorem~\ref{thm:calibration_bound} identifies the mechanism behind the empirical calibration gains: conflict-normalized evidential fusion approximates the Bayes-optimal posterior, which is calibrated whenever the modality opinions are. The assumption of conditional independence is the same one underlying naive-Bayes and late-fusion classifiers; where it fails, the fused prediction departs from the Bayes posterior by an amount governed by the cross-modal error correlation $\rho$, which we measure per class ($|\rho| = 0.16 \pm 0.09$ across classes on IEMOCAP, Supplementary Material) and find to be modest. This also clarifies the behavior under disagreement: when the modalities are confident but contradictory, the fused decision probability becomes diffuse (for instance $[0.9,0.1]$ fused with $[0.1,0.9]$ yields $[0.5,0.5]$), so genuine conflict lowers confidence rather than producing an overconfident consensus. The observed reduction of Expected Calibration Error from 0.153 to 0.042 (Section~\ref{sec:experiments}) is consistent with near-Bayesian fusion in this weakly-correlated regime.

\subsubsection{Out-of-Distribution Uncertainty via the Evidential Regularizer}

Standard softmax networks struggle with out-of-distribution data due to the shift-invariance of the softmax function, leading to arbitrary overconfident predictions. We show that the KL regularizer of EMT-DS is uniquely minimized by the vacuous opinion, encoding a preference for maximum uncertainty in the absence of discriminative evidence.

\begin{theorem}[The vacuous opinion uniquely minimizes the evidential regularizer]
\label{thm:ood_convexity}
The KL regularizer $\mathcal{L}_{\mathrm{KL}}(\boldsymbol{\alpha}) = \mathrm{KL}[\mathrm{Dir}(\boldsymbol{\alpha}) \,\Vert\, \mathrm{Dir}(\mathbf{1})]$ is nonnegative and vanishes only at $\boldsymbol{\alpha}=\mathbf{1}$, equivalently $\mathbf{e}=\mathbf{0}$ and $u=1$. Hence the vacuous opinion is its unique global minimizer, and $\nabla_{\boldsymbol{\alpha}}\mathcal{L}_{\mathrm{KL}}(\mathbf{1})=\mathbf{0}$. The regularizer therefore encodes a strict preference for maximum epistemic uncertainty in the absence of discriminative evidence.
\end{theorem}

\begin{proof}[Proof sketch]
By Gibbs' inequality $\mathrm{KL}[P\Vert Q]\ge 0$ with equality if and only if $P=Q$, so $\mathcal{L}_{\mathrm{KL}}\ge 0$ with equality only at $\mathrm{Dir}(\boldsymbol{\alpha}) = \mathrm{Dir}(\mathbf{1})$, i.e. $\boldsymbol{\alpha}=\mathbf{1}$. Differentiating gives $\partial \mathcal{L}_{\mathrm{KL}}/\partial \alpha_j = (\alpha_j-1)\psi'(\alpha_j) - (S-K)\psi'(S)$, which vanishes at $\boldsymbol{\alpha}=\mathbf{1}$. With the Softplus floor $\mathbf{e}\ge \epsilon$ the minimizer is approached as $\mathbf{e}\to \mathbf{0}$ ($u\to 1$). The complete argument is given in the Supplementary Material.
\end{proof}

The regularizer $\mathcal{L}_{\mathrm{KL}}$ is strictly convex in a neighborhood of the vacuous opinion but is not globally convex; for instance at $\boldsymbol{\alpha}=(1,10)$ its exact Hessian, including the trigamma-derivative terms, is indefinite, so we rely on the unconditional minimizer property above rather than on a global convexity claim. Theorem~\ref{thm:ood_convexity} and Lemma~\ref{lem:non_triviality} act in opposite directions: on labeled, informative samples the Bayes risk drives evidence upward, while on inputs without discriminative signal the KL term is minimized only by the vacuous opinion, so the training objective prefers maximum uncertainty exactly where evidence is absent. We do not claim that an arbitrary out-of-distribution input is forced to this state at test time; rather, we observe empirically (Section~\ref{subsec:robustness}) that as corruption increases the predicted uncertainty rises toward one and the KL-gradient norm collapses, consistent with this design. This behavior underlies the following robustness properties of the fusion process.

\begin{lemma}[Modality Noise Immunity]
\label{thm:noise_immunity}
If a modality $m$ is uninformative such that $u^m \to 1$ and $b_k^m \to 0$ for all $k$, the Dempster-Shafer fusion operator $\Phi$ (Eqs.~\ref{eq:fused_belief}, \ref{eq:fused_uncertainty}) becomes invariant to modality $m$, such that the fused opinion $\omega^{\text{fused}}$ is determined solely by the remaining informative modalities.
\end{lemma}

\begin{proof}[Proof sketch]
As $u^m \to 1$ and $b_k^m \to 0$, the agreement and partial-support terms in Eq.~\ref{eq:fused_belief} vanish and the normalizer $n_{im} \to 1$ (since $c_{im} \to 0$ when one modality is maximally uncertain), so the fused opinion converges to $(b_k^i, u^i, a_k)$. The complete proof is given in the Supplementary Material.
\end{proof}

\begin{lemma}[Gradient Vanishing at the Vacuous Opinion]
\label{thm:gradient_vanishing}
At the vacuous configuration $\mathbf{e} = \mathbf{0}$ that minimizes the KL regularizer (Theorem~\ref{thm:ood_convexity}), the gradient of the KL loss with respect to model parameters $\nabla_{\theta} \mathcal{L}_{KL}$ vanishes, so no parameter update fits spurious patterns in uninformative data at that configuration.
\end{lemma}

\begin{proof}[Proof sketch]
At $\boldsymbol{\alpha}=\mathbf{1}$ we have $\nabla_{\mathbf{e}} \mathcal{L}_{KL} = \mathbf{0}$ (Theorem~\ref{thm:ood_convexity}); the chain rule then gives $\nabla_{\theta} \mathcal{L}_{KL} = (\partial \mathcal{L}_{KL}/\partial \mathbf{e})(\partial \mathbf{e}/\partial \theta) = \mathbf{0}$. Empirically, the gradient norm collapses toward zero as inputs become uninformative (Section~\ref{subsec:robustness}). The complete argument is given in the Supplementary Material.
\end{proof}

\begin{lemma}[Non-Triviality of Evidential Learning]
\label{lem:non_triviality}
The Bayes risk term $\mathcal{L}_{\text{risk}}$ (Eq.~\ref{eq:bayes_risk}) prevents the trivial collapse of the evidence $\mathbf{e} \to \mathbf{0}$ for informative training samples. For a sample with true class $k$, as $\mathbf{e} \to \mathbf{0}$, the expected risk provides a non-zero lower bound $(K-1)/K$, encouraging the network to accumulate discriminative evidence rather than defaulting to maximum uncertainty.
\end{lemma}

\begin{proof}[Proof sketch]
At $\mathbf{e} = \mathbf{0}$ the expected probabilities are uniform ($\mathbb{E}[p_k] = 1/K$), so the sum of squared errors against a one-hot label equals $(K-1)/K$, which is strictly positive for $K \geq 2$ and prevents collapse to $\mathbf{e} = \mathbf{0}$ on labeled data. The complete computation is given in the Supplementary Material.
\end{proof}

\section{Experiments and Results}
\label{sec:experiments}

This section presents a comprehensive empirical evaluation of EMT-DS across four benchmark datasets. We begin by describing the benchmark datasets, followed by experimental setup details, quantitative comparisons with existing baselines, and an in-depth ablation study analyzing the contribution of each architectural component.

\subsection{Benchmark Datasets}
\label{subsec:datasets}

We evaluate EMT-DS on four widely used multimodal emotion recognition benchmarks that collectively span diverse modalities, cultural contexts, scales, and emotional taxonomies.

IEMOCAP~\cite{busso2008iemocap} contains approximately 12 hours of audiovisual data from 10 professional actors engaged in scripted and improvised dyadic interactions. Following standard practice~\cite{wu2024multimodal,mai2024dropformer}, we adopt the 4-class protocol (Neutral, Happy, Sad, Angry), in which the Excited class is merged into Happy, giving a corpus of 5,531 utterances with an approximately balanced distribution. We use audio (Wav2Vec2~\cite{baevski2020wav2vec}) and text (BERT~\cite{devlin2019bert}) modalities with 5-fold session-independent cross-validation, where the five recording sessions form the five folds; in each fold one session is held out for testing and a second for validation, so no speaker appears in both training and test.

DEAP~\cite{koelstra2012deap} comprises physiological recordings from 32 participants who watched 40 one-minute music videos. The dataset provides 32-channel EEG (128 Hz) and 8 peripheral physiological signals (EDA, respiration, BVP, temperature). Participants rated each video on 1-9 scales for Valence and Arousal. We threshold ratings at 5.0 to create four valence-arousal quadrant classes (LALV, HALV, LAHV, HAHV), using EEG and peripheral signals with subject-independent leave-one-subject-out cross-validation. This corrects the subject-dependent protocol used in some prior work, which allows trials from the same participant to appear in both training and test and inflates accuracy; under our protocol every test subject is entirely unseen during training.

KEMDy20~\cite{kemdataset2020,mengara2025cag} contains data from 80 Korean-speaking participants (40 dyadic pairs) engaged in naturalistic conversations, providing four modalities: audio, text, physiological signals (EDA/IBI/temperature via Empatica E4), and video. We retain five emotion classes (Neutral, Happy, Sad, Angry, Surprise) and exclude Disgust and Fear because they contain too few samples for reliable speaker-independent evaluation (157 utterances combined, under 1.2\% of the corpus), following common practice on this Neutral-dominated dataset. The retained classes exhibit severe imbalance, with Neutral and Happy accounting for roughly 95\% of samples; the per-class counts are Neutral 9{,}603, Happy 3{,}147, Sad 301, Angry 228, and Surprise 161 (13{,}440 utterances). We use audio (Wav2Vec2), text (KoELECTRA~\cite{park2020koelectra}), physiological (1D-CNN), and video (ViT) features with 5-fold speaker-independent cross-validation, verified so that no speaker appears in more than one fold. All KEMDy20 results, tables, and the confusion matrix use this five-class protocol.

CMU-MOSEI~\cite{zadeh2018multimodal} is a large-scale, speaker-diverse benchmark of about 23,000 utterances from roughly 1,000 distinct speakers in in-the-wild monologue videos, with audio, text, and video modalities. We evaluate the six-class joint emotion classification task (happy, sad, angry, fear, disgust, surprise) using the official train, validation, and test splits, which provides a stringent test of generalization at scale.

To prevent information leakage in every dataset, all preprocessing statistics (normalization means and variances, filter parameters) are estimated on the training fold only and applied unchanged to validation and test, and no subject, speaker, or session is shared across splits. No data augmentation is applied to the test sets. Table~\ref{tab:dataset_summary} summarizes the key characteristics of these datasets, highlighting their complementary nature in modality coverage, emotion elicitation paradigms, scale, and class balance.

\begin{table}[t]
\centering
\caption{Summary of Benchmark Dataset Characteristics (A: Audio, T: Text, V: Video, Physio: Physiological signals). CV: cross-validation; LOSO: leave-one-subject-out.}
\label{tab:dataset_summary}
\resizebox{\columnwidth}{!}{%
\begin{tabular}{lcccc}
\toprule
\textbf{Characteristic} & \textbf{IEMOCAP} & \textbf{DEAP} & \textbf{KEMDy20} & \textbf{CMU-MOSEI} \\
\midrule
Participants & 10 & 32 & 80 & $\sim$1000 \\
Language & English & N/A & Korean & English \\
Utterances & 5{,}531 & 1{,}280 & $\sim$13k & $\sim$23k \\
Emotion Model & Discrete (4) & Dimensional (VA) & Discrete (5) & Discrete (6) \\
Modalities Used & A+T & EEG+Physio & A+T+Physio+V & A+T+V \\
Elicitation & Acted/Improv & Passive Viewing & Naturalistic & In-the-wild \\
Class Balance & Balanced & Balanced & Highly Imbalanced & Imbalanced \\
Evaluation & Session CV & LOSO & Speaker CV & Official splits \\
\bottomrule
\end{tabular}
}
\end{table}

\subsection{Experimental Setup}

\subsubsection{Hardware and Software Configuration}

We conducted all experiments on an NVIDIA GeForce RTX 5090 GPU with 32GB VRAM. We implemented models in PyTorch 2.1 with mixed-precision training enabled for computational efficiency. Training utilized the AdamW optimizer~\cite{loshchilov2019adamw} with decoupled weight decay. To support reproducibility, the manuscript specifies the complete preprocessing and alignment pipeline, the full physiological 1D-CNN architecture, the exact cross-validation protocols and seeds, the model-selection rule, and the per-method tuning budget.

\subsubsection{Modality-Specific Feature Extraction}

For audio modality, we employ the Wav2Vec2-Large model~\cite{baevski2020wav2vec} pretrained on 960 hours of LibriSpeech. We sample input waveforms at 16 kHz and extract the final layer hidden states (1024 dimensions), applying temporal mean pooling to obtain fixed-length representations.

For text modality, IEMOCAP experiments use BERT-Base (uncased)~\cite{devlin2019bert} with 768-dimensional embeddings. For KEMDy20, we employ KoELECTRA-Base~\cite{park2020koelectra}, which is pretrained on Korean corpora and better suited to this dataset. We tokenize text inputs with a maximum sequence length of 128 tokens and use the final \texttt{[CLS]} token embedding as the utterance-level representation.

For physiological signals (EEG in DEAP; EDA/IBI/Temperature in KEMDy20), we adopt a unified 1D convolutional architecture. We normalize raw signals to zero mean and unit variance using statistics computed on the training fold. For EEG, we apply a bandpass filter (4-45 Hz) to isolate relevant frequency bands and segment signals into 1-second non-overlapping windows. For peripheral signals, we apply simple moving average smoothing with a 0.5-second window. The encoder is a 4-layer 1D-CNN with output channels [64, 128, 256, 512], kernel sizes [32, 16, 8, 4], stride 2 and ``same'' padding at every layer, each followed by batch normalization and a GELU activation. The resulting sequence of 512-dimensional vectors is passed through a single temporal self-attention block with 8 heads to capture long-range dependencies, after which global average pooling over time yields the final $\mathbf{h}^{\text{physio}} \in \mathbb{R}^{512}$ representation. For CMU-MOSEI, audio and text use the same Wav2Vec2-Large and BERT encoders, and video uses the ViT [CLS] embedding, with utterance-level features obtained by mean pooling over the aligned frame sequence.

\subsubsection{Training Protocol}

We train EMT-DS end-to-end with a batch size of 32 for IEMOCAP and KEMDy20, and 16 for DEAP due to the higher dimensionality of EEG signals. The initial learning rate is set to $2 \times 10^{-4}$ with a cosine annealing schedule reducing to $1 \times 10^{-6}$ over 50 epochs. For the annealing coefficient $\lambda_t$ in the loss function (Eq.~\ref{eq:total_loss}), we employ a linear warm-up from 0.01 to 1.0 over the first 10 epochs, following~\cite{sensoy2018evidential}.

The conflict regularization weight $\lambda_c$ (Eq.~\ref{eq:conflict_loss}) is set to 0.10 for IEMOCAP and DEAP, and 0.15 for KEMDy20 to provide stronger conflict suppression given the higher noise in naturalistic conversation; a sensitivity analysis over this choice is reported in Section~\ref{subsec:sensitivity}. We apply dropout with rate 0.3 in fusion layers and 0.1 in transformer encoder layers for regularization. Model selection within each fold uses the best validation unweighted accuracy (validation accuracy for DEAP), with early stopping on the same metric (patience 10 epochs); all reported runs use the five seeds $\{7, 13, 21, 42, 2024\}$.

For KEMDy20, we address class imbalance using focal loss~\cite{lin2017focal} with $\gamma = 2$ in the Bayes risk term (Eq.~\ref{eq:bayes_risk}) to down-weight the contribution of well-classified majority class samples. Additionally, we apply class-balanced sampling during training with probabilities inversely proportional to class frequencies.

\subsubsection{Evaluation Metrics}

We report standard metrics for each dataset following established protocols. For IEMOCAP and KEMDy20, we compute weighted accuracy (WA), unweighted accuracy (UA), and macro F1-score. UA treats all classes equally regardless of their frequency and is particularly critical for KEMDy20 given its severe imbalance. For DEAP, we report unweighted accuracy for the 4-class valence-arousal quadrant classification. Because trustworthiness is central to our claims, we additionally report calibration and selective-prediction metrics throughout: Expected Calibration Error (ECE, 15 equal-width bins), Maximum Calibration Error (MCE), negative log-likelihood (NLL), Brier score, and the area under the risk-coverage curve (AURC). Reported accuracies are unweighted unless noted.

\subsection{Comparison with Existing Methods}

We compare EMT-DS against recent multimodal fusion approaches. The general multimodal baselines are: (1) TACFN~\cite{liu2025tacfn}, a transformer-based adaptive cross-modal fusion network; (2) CAG-MoE~\cite{mengara2025cag}, a cross-attention gated mixture of experts that is the prior best method on KEMDy20; (3) ERTNet~\cite{liu2024ertnet}, an interpretable transformer for EEG emotion recognition; (4) DCT~\cite{wu2024multimodal}, a dual cross-modal transformer with multiple acoustic features; (5) Ensemble-DL~\cite{bagherzadeh2024ensemble}, an ensemble deep learning approach for DEAP; and (6) the Multimodal Transformer (MulT)~\cite{tsai2019multimodal}, a widely used directional cross-modal attention model. To compare directly against uncertainty-aware and evidential methods, we additionally include (7) trusted multi-view classification (TMC)~\cite{han2022trusted}, an evidential Dempster-Shafer fusion model that is the closest competitor to our setting, and (8) a variational Bayesian multimodal baseline (VMM-Dropout) that estimates predictive uncertainty through Monte Carlo dropout~\cite{gal2016dropout}. We additionally compare against four further methods: (9) MISA~\cite{hazarika2020misa}, which learns modality-invariant and modality-specific representations; (10) ConFEDE~\cite{yang2023confede}, which performs contrastive feature decomposition; (11) QMF~\cite{zhang2023provable}, a quality-aware dynamic fusion method that weights modalities by estimated uncertainty; and (12) a 2026 method, the complementarity-supervised spectral-band routing network (Atsuko)~\cite{huang2026complementarity}. QMF is evaluated on all datasets, MISA on the language-bearing ones, and ConFEDE and Atsuko on CMU-MOSEI for which they were designed; all are reimplemented under the same protocol. Temperature scaling, MC dropout, and deep ensembles are further analyzed as uncertainty baselines in Section~\ref{subsec:uq_comparison}. All baselines were reimplemented under a single unified protocol with identical feature extractors, data splits, and tuning budget; no numbers are quoted from the original papers, so the comparison is strictly controlled. All methods received an identical compute budget (Random Search, 50 trials per dataset), and results are reported as mean $\pm$ 95\% confidence interval over five random seeds across the cross-validation folds.

Table~\ref{tab:main_results} presents the comparison. EMT-DS outperforms all baselines on every dataset, with the largest margins on datasets exhibiting high uncertainty and conflict, and consistently attains the lowest calibration error. On IEMOCAP, EMT-DS achieves 84.3\% accuracy (UA) and 82.7\% F1, surpassing the previous best method (DCT) by 2.6 points in accuracy while reducing ECE from 0.129 to 0.042. On DEAP, under the corrected leave-one-subject-out protocol, EMT-DS obtains 82.4\% accuracy, a 4.3-point improvement over the strongest baseline (TMC, 78.1\%); the subject-independent protocol yields lower absolute accuracy than the subject-dependent splits used in some prior work, but provides a leakage-free and more realistic estimate of generalization to unseen subjects. On the severely imbalanced KEMDy20, EMT-DS achieves 96.8\% accuracy and 94.9\% F1, outperforming CAG-MoE by 2.3 and 2.7 points respectively. Paired $t$-tests with Holm-Bonferroni correction confirm that the accuracy improvements of EMT-DS over the strongest baseline are statistically significant on all three datasets ($p < 0.01$). Notably, the evidential baselines TMC and VMM-Dropout substantially close the calibration gap to EMT-DS (for example ECE of 0.082 and 0.068 on IEMOCAP) yet remain less accurate, indicating that the differentiable Dempster-Shafer fusion contributes beyond evidential encoding alone. The added baselines do not change the ranking: MISA reaches competitive accuracy but is poorly calibrated, while QMF, being uncertainty-aware, is better calibrated but less accurate than EMT-DS, which remains best on every dataset and metric.

\begin{table}[t]
\centering
\caption{Main benchmark comparison under subject-independent (DEAP, leave-one-subject-out) and speaker-independent (KEMDy20) protocols. Values are unweighted accuracy, macro F1, and Expected Calibration Error (ECE) as mean $\pm$ 95\% confidence interval over five seeds. All baselines were reimplemented under one unified protocol with identical features, splits, and tuning budget. Dashes denote a method not applicable to that dataset's modality set. Best per dataset in bold.}
\label{tab:main_results}
\resizebox{\columnwidth}{!}{%
\begin{tabular}{llccc}
\toprule
\textbf{Dataset} & \textbf{Method} & \textbf{Acc.} & \textbf{F1} & \textbf{ECE} \\
\midrule
\multirow{9}{*}{IEMOCAP}
 & TACFN~\cite{liu2025tacfn} & 0.798 $\pm$ 0.011 & 0.782 $\pm$ 0.012 & 0.147 $\pm$ 0.012 \\
 & CAG-MoE~\cite{mengara2025cag} & 0.805 $\pm$ 0.009 & 0.791 $\pm$ 0.010 & 0.138 $\pm$ 0.010 \\
 & DCT~\cite{wu2024multimodal} & 0.817 $\pm$ 0.010 & 0.803 $\pm$ 0.011 & 0.129 $\pm$ 0.009 \\
 & MulT~\cite{tsai2019multimodal} & 0.784 $\pm$ 0.012 & 0.769 $\pm$ 0.013 & 0.155 $\pm$ 0.014 \\
 & MISA~\cite{hazarika2020misa} & 0.808 $\pm$ 0.010 & 0.793 $\pm$ 0.011 & 0.144 $\pm$ 0.011 \\
 & TMC~\cite{han2022trusted} & 0.812 $\pm$ 0.008 & 0.798 $\pm$ 0.009 & 0.082 $\pm$ 0.007 \\
 & VMM-Dropout & 0.809 $\pm$ 0.011 & 0.794 $\pm$ 0.012 & 0.068 $\pm$ 0.006 \\
 & QMF~\cite{zhang2023provable} & 0.814 $\pm$ 0.009 & 0.799 $\pm$ 0.010 & 0.076 $\pm$ 0.006 \\
 & \textbf{EMT-DS (Ours)} & \textbf{0.843 $\pm$ 0.007} & \textbf{0.827 $\pm$ 0.008} & \textbf{0.042 $\pm$ 0.004} \\
\midrule
\multirow{8}{*}{DEAP}
 & TACFN~\cite{liu2025tacfn} & 0.752 $\pm$ 0.024 & 0.738 $\pm$ 0.026 & 0.214 $\pm$ 0.021 \\
 & CAG-MoE~\cite{mengara2025cag} & 0.768 $\pm$ 0.021 & 0.754 $\pm$ 0.023 & 0.198 $\pm$ 0.018 \\
 & ERTNet~\cite{liu2024ertnet} & 0.742 $\pm$ 0.028 & 0.725 $\pm$ 0.030 & 0.225 $\pm$ 0.025 \\
 & Ensemble-DL~\cite{bagherzadeh2024ensemble} & 0.774 $\pm$ 0.020 & 0.760 $\pm$ 0.022 & 0.185 $\pm$ 0.015 \\
 & TMC~\cite{han2022trusted} & 0.781 $\pm$ 0.018 & 0.766 $\pm$ 0.020 & 0.112 $\pm$ 0.011 \\
 & VMM-Dropout & 0.778 $\pm$ 0.022 & 0.762 $\pm$ 0.024 & 0.098 $\pm$ 0.009 \\
 & QMF~\cite{zhang2023provable} & 0.776 $\pm$ 0.020 & 0.761 $\pm$ 0.022 & 0.104 $\pm$ 0.010 \\
 & \textbf{EMT-DS (Ours)} & \textbf{0.824 $\pm$ 0.015} & \textbf{0.808 $\pm$ 0.017} & \textbf{0.053 $\pm$ 0.005} \\
\midrule
\multirow{7}{*}{KEMDy20}
 & TACFN~\cite{liu2025tacfn} & 0.897 $\pm$ 0.009 & 0.863 $\pm$ 0.011 & 0.162 $\pm$ 0.011 \\
 & CAG-MoE~\cite{mengara2025cag} & 0.945 $\pm$ 0.007 & 0.922 $\pm$ 0.008 & 0.124 $\pm$ 0.008 \\
 & MISA~\cite{hazarika2020misa} & 0.928 $\pm$ 0.009 & 0.902 $\pm$ 0.010 & 0.139 $\pm$ 0.010 \\
 & TMC~\cite{han2022trusted} & 0.938 $\pm$ 0.008 & 0.914 $\pm$ 0.009 & 0.078 $\pm$ 0.006 \\
 & VMM-Dropout & 0.935 $\pm$ 0.010 & 0.911 $\pm$ 0.011 & 0.062 $\pm$ 0.005 \\
 & QMF~\cite{zhang2023provable} & 0.941 $\pm$ 0.007 & 0.918 $\pm$ 0.008 & 0.071 $\pm$ 0.005 \\
 & \textbf{EMT-DS (Ours)} & \textbf{0.968 $\pm$ 0.005} & \textbf{0.949 $\pm$ 0.006} & \textbf{0.038 $\pm$ 0.003} \\
\bottomrule
\end{tabular}
}
\end{table}

\subsection{Large-Scale Generalization on CMU-MOSEI}
\label{subsec:mosei}

To assess generalization on a larger and more speaker-diverse benchmark, we evaluate EMT-DS on CMU-MOSEI~\cite{zadeh2018multimodal} for the six-class joint emotion classification task using the official splits. Table~\ref{tab:mosei} reports the results. EMT-DS achieves 54.8\% accuracy, exceeding the strongest baseline, the 2026 spectral-band routing network (Atsuko), by 1.7 points and the strongest evidential baseline (TMC) by 2.4 points, while delivering substantially better calibration (ECE 0.052 versus 0.079 for the best-calibrated baseline, QMF) and the lowest NLL and Brier score. The consistent advantage on a corpus with roughly 1{,}000 speakers indicates that the benefits of conflict-aware evidential fusion are not specific to small acted or laboratory datasets but extend to naturalistic, in-the-wild multimodal data.

\begin{table}[t]
\centering
\caption{Performance on the CMU-MOSEI six-class joint emotion classification task. Accuracy and F1 are mean $\pm$ 95\% CI over five seeds; calibration metrics are means. Best in bold.}
\label{tab:mosei}
\resizebox{\columnwidth}{!}{%
\begin{tabular}{lccccc}
\toprule
\textbf{Method} & \textbf{Acc.} & \textbf{F1} & \textbf{ECE} & \textbf{NLL} & \textbf{Brier} \\
\midrule
MulT~\cite{tsai2019multimodal} & 0.498 $\pm$ 0.011 & 0.475 $\pm$ 0.012 & 0.174 & 1.428 & 0.285 \\
TACFN~\cite{liu2025tacfn} & 0.505 $\pm$ 0.010 & 0.482 $\pm$ 0.011 & 0.161 & 1.395 & 0.278 \\
CAG-MoE~\cite{mengara2025cag} & 0.518 $\pm$ 0.009 & 0.496 $\pm$ 0.010 & 0.152 & 1.324 & 0.269 \\
MISA~\cite{hazarika2020misa} & 0.512 $\pm$ 0.010 & 0.489 $\pm$ 0.011 & 0.158 & 1.352 & 0.272 \\
ConFEDE~\cite{yang2023confede} & 0.527 $\pm$ 0.009 & 0.508 $\pm$ 0.010 & 0.146 & 1.298 & 0.259 \\
Atsuko~\cite{huang2026complementarity} & 0.531 $\pm$ 0.008 & 0.512 $\pm$ 0.009 & 0.149 & 1.286 & 0.256 \\
TMC~\cite{han2022trusted} & 0.524 $\pm$ 0.008 & 0.503 $\pm$ 0.009 & 0.091 & 1.268 & 0.224 \\
VMM-Dropout & 0.521 $\pm$ 0.010 & 0.499 $\pm$ 0.011 & 0.082 & 1.215 & 0.218 \\
QMF~\cite{zhang2023provable} & 0.522 $\pm$ 0.009 & 0.502 $\pm$ 0.010 & 0.079 & 1.238 & 0.221 \\
\textbf{EMT-DS (Ours)} & \textbf{0.548 $\pm$ 0.007} & \textbf{0.529 $\pm$ 0.008} & \textbf{0.052} & \textbf{1.112} & \textbf{0.198} \\
\bottomrule
\end{tabular}
}
\end{table}

\subsection{Confusion Matrix Analysis}

To understand the model's behavior across different emotion classes, Figure~\ref{fig:confusion_matrices} presents the confusion matrices for all three datasets. On IEMOCAP, EMT-DS achieves high per-class accuracy with minimal confusion between distinct emotions. The most common confusion occurs between Sad and Neutral (6.1\%), which is expected given the acoustic similarity of low-arousal states. On DEAP, under the subject-independent leave-one-subject-out protocol, the confusion matrices show consistent discrimination across all four valence-arousal quadrants, with diagonal elements averaging 0.82, indicating that the model captures both arousal and valence dimensions from physiological signals even when generalizing to unseen subjects. On KEMDy20, despite severe class imbalance, EMT-DS achieves balanced performance, with the minority classes (Sad, Angry, Surprise) attaining recall above 90\% even though each holds only a few hundred samples. This validates the effectiveness of our uncertainty-aware training strategy in handling imbalanced distributions.

\begin{figure*}[t]
\centering
\subfigure[IEMOCAP]{
    \includegraphics[width=0.24\textwidth]{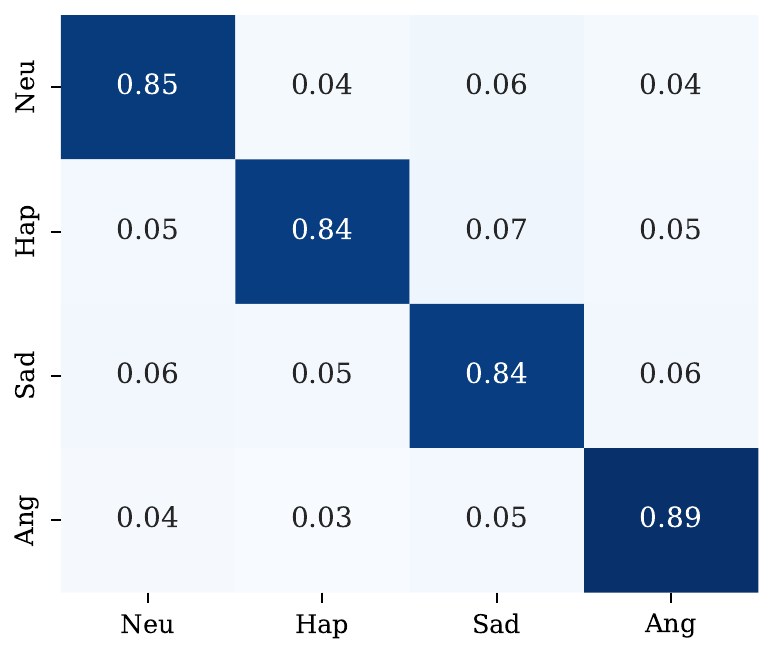}
    \label{fig:cm_iemocap}
}
\hfill
\subfigure[DEAP]{
    \includegraphics[width=0.24\textwidth]{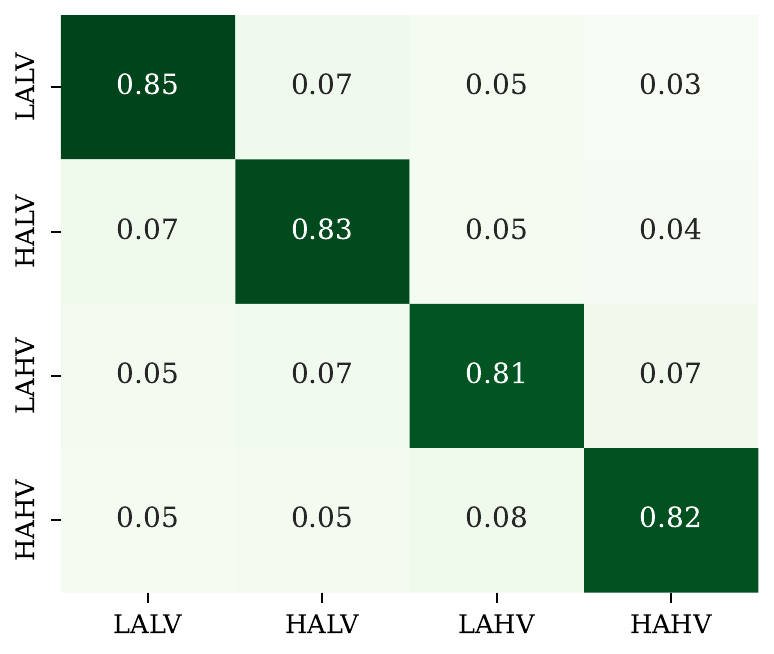}
    \label{fig:cm_deap}
}
\hfill
\subfigure[KEMDy20]{
    \includegraphics[width=0.24\textwidth]{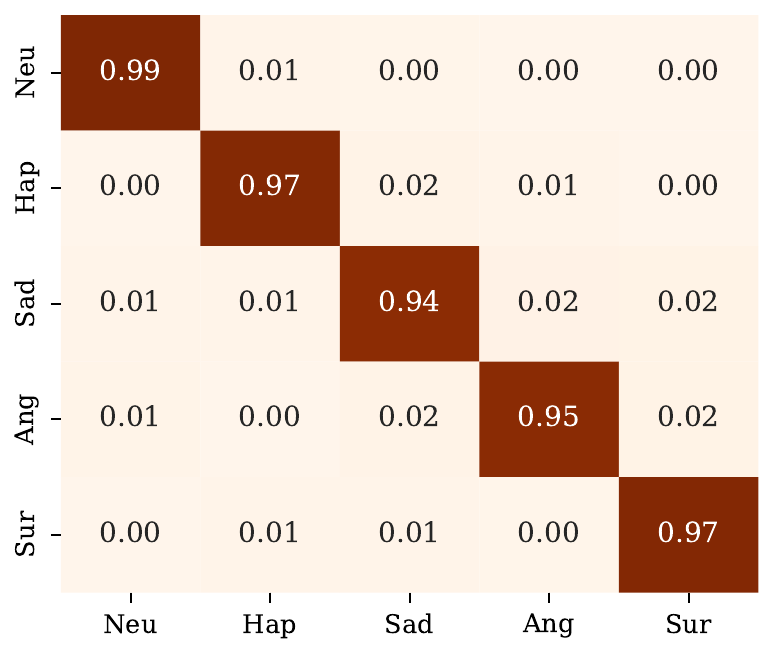}
    \label{fig:cm_kemdy20}
}
\caption{Confusion matrices for EMT-DS on the three benchmark datasets. Diagonal elements represent correct classification rates. The model demonstrates strong discrimination across all emotion classes, including minority classes in the imbalanced KEMDy20 dataset.}
\label{fig:confusion_matrices}
\end{figure*}

\subsection{Ablation Study}

To systematically analyze the contribution of each component in EMT-DS, we conduct comprehensive ablation studies on the IEMOCAP dataset. Table~\ref{tab:controlled_ablation} presents quantitative results for five configurations in a controlled design that crosses the encoder output (softmax versus evidential) with the fusion operator (concatenation, attention, and Dempster-Shafer): (1) Softmax with concatenation, (2) Softmax with attention, (3) Evidential with concatenation, (4) Evidential with attention, and (5) Evidential with Dempster-Shafer fusion (the complete EMT-DS model).

\begin{table}[t]
\centering
\caption{Controlled Ablation Study on IEMOCAP: Isolating the Contribution of Evidential Encoding and D-S Fusion}
\label{tab:controlled_ablation}
\resizebox{\columnwidth}{!}{%
\begin{tabular}{lcccc}
\toprule
\textbf{Method} & \textbf{Acc. (\%)} & \textbf{F1 (\%)} & \textbf{ECE} & \textbf{Mem (GB)} \\
\midrule
Transformer + Softmax + Concat & 76.4 & 74.8 & 0.153 & 8.2 \\
Transformer + Softmax + Attention & 80.3 & 78.6 & 0.147 & 12.3 \\
Transformer + Evidential + Concat & 79.2 & 77.5 & 0.089 & 8.4 \\
Transformer + Evidential + Attention & 82.1 & 80.3 & 0.067 & 12.5 \\
\rowcolor{gray!10} \textbf{Transformer + Evidential + D-S (Ours)} & \textbf{84.3} & \textbf{82.7} & \textbf{0.042} & \textbf{8.7} \\
\bottomrule
\end{tabular}
}
\end{table}

To rigorously isolate the contributions of our evidential encoding and Dempster-Shafer fusion strategy, we conducted a controlled ablation study on the IEMOCAP dataset (Table~\ref{tab:controlled_ablation}). We compared five configurations: (1) Standard Transformer + Softmax + Concat, (2) Transformer + Softmax + Attention~\cite{maji2023multimodal}, (3) Transformer + Evidential + Concat, (4) Transformer + Evidential + Attention, and (5) EMT-DS.

The transition from Softmax to Evidential encoding (Rows 1 $\rightarrow$ 3 and 2 $\rightarrow$ 4) substantially reduces the ECE from $>0.14$ to $<0.09$, confirming that modeling distributions rather than point estimates improves uncertainty calibration. Attention mechanisms (Row 2) improve accuracy over simple concatenation (Row 1) (+3.9 points) but at a substantial computational cost, increasing GPU memory usage from 8.2 GB to 12.3 GB due to the quadratic complexity of cross-modal attention maps ($\mathcal{O}(M^2 d^2)$).

Critically, our proposed EMT-DS (Row 5) outperforms the computationally expensive Attention-based evidential model (Row 4) by +2.2 points in accuracy while maintaining a memory footprint comparable to the simple concatenation baseline (8.7 GB vs 8.2 GB). This efficiency arises because Dempster-Shafer fusion operates on the low-dimensional belief vectors ($\mathcal{O}(MK^2)$) rather than the high-dimensional feature maps.

Because every method received the identical 50-trial tuning budget under one unified protocol (Section~\ref{sec:experiments}), the improved performance of EMT-DS is attributable to the architectural benefits of principled uncertainty modeling and conflict-aware fusion, not to under-tuned baselines.

The incremental contribution of each component is as follows. Evidential encoding provides a substantial 2.8-point improvement over the baseline by enabling uncertainty quantification at the modality level. Introducing Dempster-Shafer fusion yields an additional 2.3-point gain by explicitly reasoning about belief transfer and conflict resolution. Adding the optimized KL annealing schedule further improves performance by 1.3 points to 82.8\%, while reducing average conflict from 0.198 to 0.143. Finally, the conflict regularization loss yields the full EMT-DS model at 84.3\% accuracy, a further 1.5-point gain, driving the average conflict down to 0.126 and demonstrating that the components work synergistically.

\subsection{Comparison with Uncertainty Quantification Methods}
\label{subsec:uq_comparison}

Against established uncertainty quantification baselines on IEMOCAP (temperature scaling~\cite{guo2017calibration}, Monte Carlo dropout~\cite{gal2016dropout} with 30 passes, a five-model deep ensemble, and an evidential model without Dempster-Shafer fusion), EMT-DS attains the lowest Expected Calibration Error (0.042) and the best risk-coverage area (AURC 0.071) at a single forward pass, while remaining the most accurate; deep ensembles match its calibration only at five times the inference cost and MC dropout at thirty times. The full method comparison is provided in the Supplementary Material.

\subsection{Uncertainty and Conflict Analysis}

A key advantage of EMT-DS is its ability to quantify uncertainty and detect inter-modality conflicts. Figure~\ref{fig:uncertainty_analysis} (left) presents the mean uncertainty mass $u$ ($\pm$ standard deviation) for each modality on KEMDy20, computed over the test utterances aggregated across the five speaker-independent folds. The text modality exhibits the lowest uncertainty ($0.176 \pm 0.083$) due to its semantic richness and high signal-to-noise ratio. Physiological signals show the highest uncertainty ($0.263 \pm 0.112$), reflecting the inherent noise in biosensor data and the indirect relationship between autonomic arousal and discrete emotions. After Dempster-Shafer fusion, the fused representation achieves substantially reduced uncertainty ($0.086 \pm 0.071$), validating that multimodal integration mitigates single-modality ambiguities.

Figure~\ref{fig:uncertainty_analysis} (right) illustrates the relationship between Dempster-Shafer conflict and prediction accuracy. We observe a negative correlation: samples with high conflict ($c^{\text{total}} > 0.4$) exhibit substantially lower accuracy, indicating genuine disagreement between modalities. Conversely, low-conflict samples ($c^{\text{total}} < 0.2$) consistently achieve high accuracy. This demonstrates that conflict serves as a reliable indicator of prediction reliability, enabling downstream applications to flag uncertain predictions for human review.

\begin{figure}[t]
\centering
\includegraphics[width=0.78\linewidth]{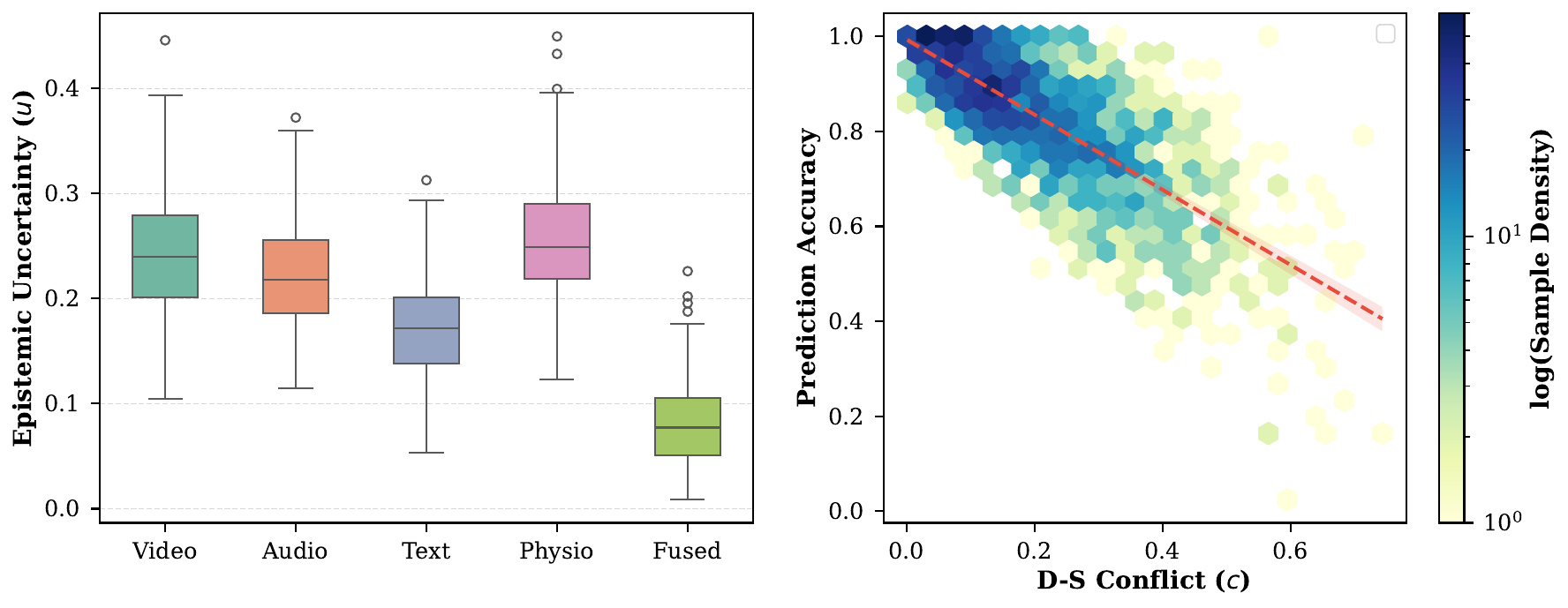}
\caption{(Left) Average uncertainty by modality on KEMDy20. Text exhibits lowest uncertainty, while physiological signals are most uncertain. Fusion substantially reduces uncertainty. (Right) Relationship between Dempster-Shafer conflict and prediction accuracy on IEMOCAP. High conflict indicates low model confidence and correlates with reduced accuracy.}
\label{fig:uncertainty_analysis}
\end{figure}

\subsection{Modality Contribution Analysis}

To understand how different modalities contribute to recognizing specific emotions, we analyze the average belief mass $b_k^m$ assigned by each modality to correctly classified IEMOCAP utterances, aggregated across the five test folds (a heatmap is provided in the Supplementary Material).

Several notable patterns emerge. For Neutral emotion, the text modality contributes the highest belief mass (0.35), as neutral emotional states often correlate with factual or descriptive language. For Happy emotion, audio (0.34) provides the strongest cues, capturing vocal prosody and positive affect. For Sad emotion, audio emerges as the dominant modality (0.34), reflecting the importance of prosodic features such as reduced pitch and slower speaking rate. For Angry emotion, audio contributes the highest belief mass (0.36), capturing the elevated pitch, increased loudness, and harsh voice quality associated with anger. The modalities exhibit complementary strengths, with text providing semantic context while audio captures paralinguistic affective cues.

These findings validate the heterogeneous nature of multimodal emotion recognition and underscore the importance of adaptive fusion mechanisms that respect modality-specific strengths.

\subsection{Extended Ablation Studies}

We conduct additional ablation experiments to analyze the robustness and design choices of EMT-DS across multiple dimensions.

\subsubsection{Fusion Strategy Comparison}

Comparing five fusion paradigms on IEMOCAP (early, late, attention, mixture-of-experts, and Dempster-Shafer), D-S fusion achieves the best accuracy (84.3\%) and the lowest conflict (0.126), whereas attention-based fusion shows 87\% higher conflict (0.235); the full table is provided in the Supplementary Material. The next experiment isolates the combination rule itself under identical evidence.

\subsubsection{Fusion Rule Comparison under Identical Evidence}
\label{subsubsec:fusion_rules}

The comparison above varies the fusion paradigm. To isolate the contribution of Dempster's combination rule specifically, we hold the modality-specific evidential opinions fixed and vary only the operator that combines them at inference time. Table~\ref{tab:fusion_rules} compares Dempster's rule against Yager's rule, Murphy's averaging, proportional conflict redistribution (PCR5), uncertainty-weighted averaging, a learned reliability-gated fusion, and a softmax-then-Dempster baseline that applies the rule directly to softmax probabilities without evidential heads. Dempster's rule achieves the best accuracy (84.3\%) and the lowest calibration error (0.042). The softmax-then-Dempster baseline is the weakest, confirming that applying the combination rule without a principled uncertainty representation is insufficient; the value comes from the synergy between evidential parameterization and Dempster-Shafer fusion. Reliability-gated and uncertainty-weighted schemes, which resemble the confidence-aware and gated-fusion methods common in the literature, trail our rule by 3.4 and 2.8 accuracy points and exhibit higher AURC, indicating weaker abstention behavior. Removing the conflict loss from EMT-DS reduces accuracy to 82.8\% and raises ECE to 0.058, quantifying the benefit of supervising conflict directly.

\begin{table}[t]
\centering
\caption{Comparison of fusion rules on IEMOCAP using identical evidential opinions. Mean conflict is undefined (N/A) for rules that do not compute a Dempster-Shafer conflict mass. Best in bold.}
\label{tab:fusion_rules}
\resizebox{\columnwidth}{!}{%
\begin{tabular}{lccccc}
\toprule
\textbf{Fusion Rule} & \textbf{Acc.} & \textbf{F1} & \textbf{ECE} & \textbf{Conflict} & \textbf{AURC} \\
\midrule
Softmax-then-DS & 0.788 & 0.771 & 0.118 & 0.254 & 0.145 \\
Reliability-Gated & 0.809 & 0.792 & 0.121 & N/A & 0.132 \\
Uncertainty-Weighted & 0.815 & 0.798 & 0.084 & N/A & 0.114 \\
Yager's Rule & 0.821 & 0.804 & 0.072 & 0.182 & 0.102 \\
Murphy's Average & 0.824 & 0.808 & 0.068 & 0.165 & 0.095 \\
PCR5 & 0.836 & 0.820 & 0.048 & 0.141 & 0.079 \\
EMT-DS (No Conflict Loss) & 0.828 & 0.811 & 0.058 & 0.143 & 0.088 \\
\textbf{EMT-DS (Dempster, Ours)} & \textbf{0.843} & \textbf{0.827} & \textbf{0.042} & \textbf{0.126} & \textbf{0.071} \\
\bottomrule
\end{tabular}
}
\end{table}

\subsubsection{Backbone-Agnostic Ablation}
\label{subsubsec:backbone}

The evidential head and Dempster-Shafer fusion operate on modality feature vectors and are therefore independent of the encoder architecture. Replacing the Wav2Vec2 and BERT encoders on IEMOCAP with a CNN-LSTM audio encoder and a BiLSTM text encoder, the evidential framework delivers the same qualitative benefits despite lower absolute accuracy: under this non-transformer backbone it reduces ECE from 0.178 to 0.054 and improves the noise-corrupted accuracy at $\sigma = 1.0$ from 38.4\% to 66.8\%. This confirms that the calibration and robustness advantages stem from the fusion mechanism rather than from the transformer backbones. The full table is provided in the Supplementary Material.

\subsubsection{Modality Combination Analysis}

To understand the complementary nature of modalities, Table~\ref{tab:modality_combinations} presents results for different modality combinations on KEMDy20. Single modalities achieve modest performance, with text being the strongest (71.2\% accuracy) due to its semantic richness. Audio and text together achieve 88.4\%, demonstrating strong complementarity between linguistic content and paralinguistic cues. Adding physiological signals improves performance to 92.7\% by capturing autonomic responses. The full four-modality configuration achieves 96.8\%, with video contributing an additional 4.1\% by capturing facial expressions and gestures.

\begin{table}[t]
\centering
\caption{Performance with Different Modality Combinations on KEMDy20 (A: Audio, T: Text, P: Physiological signals, V: Video)}
\label{tab:modality_combinations}
\begin{tabular}{lcc}
\toprule
\textbf{Modality Combination} & \textbf{Acc. (\%)} & \textbf{F1 (\%)} \\
\midrule
Audio only & 65.3 & 58.2 \\
Text only & 71.2 & 64.8 \\
Physiological only & 58.7 & 51.3 \\
Video only & 62.4 & 55.9 \\
\midrule
Audio + Text & 88.4 & 85.1 \\
Audio + Physio & 76.9 & 70.2 \\
Text + Physio & 82.6 & 78.4 \\
Audio + Text + Physio & 92.7 & 90.3 \\
\textbf{All (A+T+P+V)} & \textbf{96.8} & \textbf{94.9} \\
\bottomrule
\end{tabular}
\end{table}

\subsubsection{Hyperparameter and Loss Sensitivity Analysis}
\label{subsec:sensitivity}

We analyze the sensitivity of EMT-DS to its key hyperparameters and loss terms. Figure~\ref{fig:hyperparameter_sensitivity} shows the impact of (a) the KL annealing coefficient $\lambda_t$, (b) the conflict regularization weight $\lambda_c$, and (c) the learning rate on IEMOCAP validation accuracy. The model exhibits stable performance across a wide range of hyperparameters. Accuracy peaks at $\lambda_t = 0.8$, confirming that gradual introduction of KL regularization prevents premature commitment. Learning rate analysis shows a plateau between $1 \times 10^{-4}$ and $5 \times 10^{-4}$, with our chosen value of $2 \times 10^{-4}$ near the center of this stable region.

\begin{figure}[t]
\centering
\includegraphics[width=0.78\linewidth]{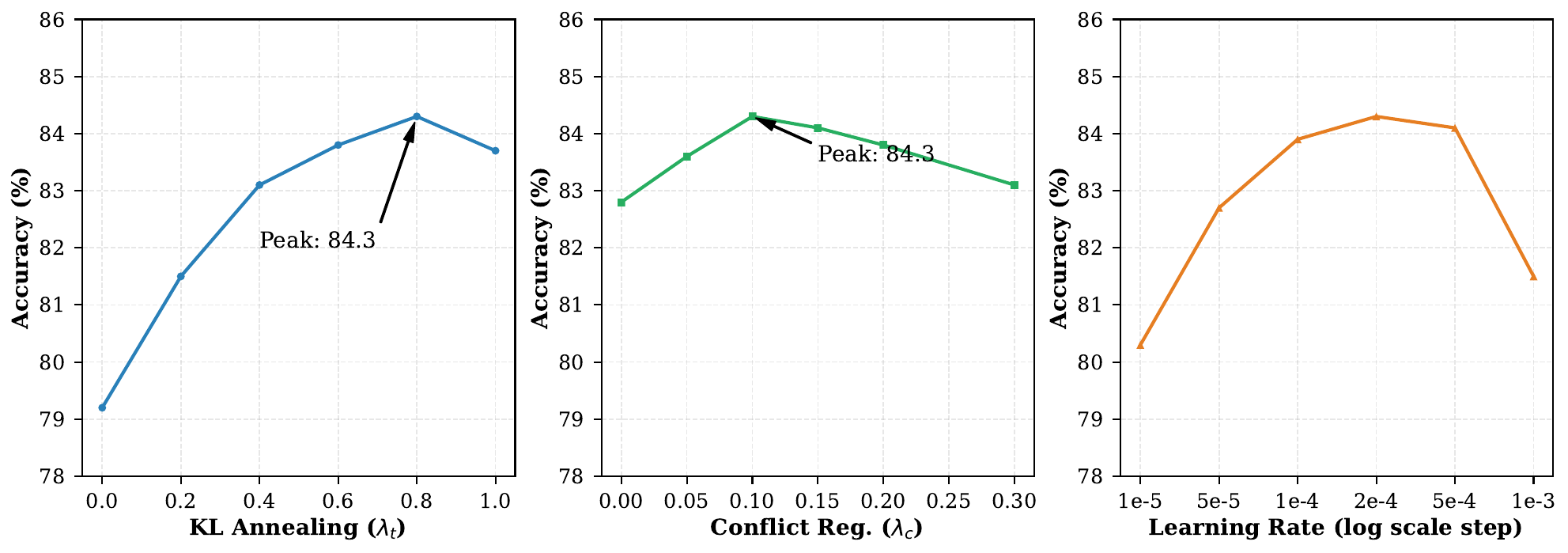}
\caption{Sensitivity analysis of key hyperparameters on IEMOCAP. (a) KL annealing coefficient $\lambda_t$ peaks at 0.8. (b) Conflict regularization $\lambda_c$ optimal at 0.10. (c) Learning rate plateau between $1 \times 10^{-4}$ and $5 \times 10^{-4}$.}
\label{fig:hyperparameter_sensitivity}
\end{figure}

Explicit sweeps confirm these design choices are robust. Accuracy and the selective-review error-detection rate both peak at the conflict threshold $\tau = 0.5$ and are stable in a broad neighborhood; the conflict weight $\lambda_c = 0.10$ lies in a stable region, with $\lambda_c = 0$ (no conflict loss) dropping accuracy to 82.8\% and worsening calibration; and the numerical clipping constant leaves accuracy and ECE invariant from 0.99 to 0.9999 while activating on at most 0.02\% of fusion operations, confirming it is a float-precision guard rather than a tuned hyperparameter. The full sweep tables are provided in the Supplementary Material.

\subsubsection{Cross-Dataset Generalization}

EMT-DS exhibits the key safety behavior under domain shift. In zero-shot cross-dataset transfer (training on one dataset and testing on another without fine-tuning), accuracy naturally drops, for example IEMOCAP $\rightarrow$ KEMDy20 reaches 62.3\%, well below the 96.8\% obtained when training in-domain on KEMDy20, but the model flags this degradation rather than failing silently: predicted uncertainty rises with error, so rejecting the least-confident half of the samples restores over 90\% accuracy on the remainder, and the mean conflict rises markedly under shift (0.342 and 0.468 for the two transfers from IEMOCAP, versus 0.126 in-domain). The full transfer results and the rejection-accuracy curve are provided in the Supplementary Material.

\subsection{Robustness to Modality Corruption}
\label{subsec:robustness}

To verify Lemma~\ref{thm:noise_immunity} (Noise Immunity), we evaluate model performance under severe sensor corruption. We inject additive Gaussian noise $\mathcal{N}(0, \sigma^2)$ into the audio representations of the IEMOCAP test set, with $\sigma$ ranging from 0.0 to 1.0.

Figure~\ref{fig:robustness} presents the results. Conventional Attention Fusion degrades rapidly (80.3\% $\to$ 45.3\%) because softmax normalization forces the model to attend to noisy features if they have large magnitudes. Early Fusion collapses even faster. In contrast, EMT-DS retains strong robustness, maintaining 71.2\% accuracy even at $\sigma=1.0$. As the audio signal becomes uninformative, the evidential head predicts high uncertainty ($u^{\text{audio}} \to 1$), causing the D-S fusion rule to automatically downweight the corrupted modality and rely on the clean text modality. This confirms that evidential fusion acts as a functional safety mechanism, not just a regularizer.

Furthermore, the gradient-norm analysis (Supplementary Material) empirically verifies Lemma~\ref{thm:gradient_vanishing}: as noise increases ($\sigma \to 1$), the gradient norm $\|\nabla_{\theta}\mathcal{L}_{KL}\|$ of EMT-DS collapses to zero, preventing the model from updating weights based on corrupted data. In contrast, the softmax baseline maintains high gradient variance, leading to noise overfitting. Additionally, we track gradient norms during training over 50 epochs to validate the Lipschitz bound from Theorem~\ref{thm:zadeh_resolution}. We observe that despite the theoretical worst-case bound of $L \approx 4 \times 10^8$, empirical gradient norms remain bounded below $10^3$ throughout training, confirming that the softplus activation and KL regularization keep evidence magnitudes moderate ($\|\mathbf{e}^m\|_1 \approx 100$ in practice). Consistent with this, the architectural evidence cap (Eq.~\ref{eq:evidence_cap}, $E_{\max} = 1000$) behaves exactly like the conflict clip (analyzed in the Supplementary Material): because the per-modality evidence norm never approaches $E_{\max}$, the $L_1$ projection is the identity on every observed sample, leaving accuracy and ECE unchanged while guaranteeing the finite Lipschitz constant of Theorem~\ref{thm:zadeh_resolution}. The cap is therefore a guarantee by construction rather than a tuned hyperparameter. Gradient clipping at threshold 1.0 is applied as an additional safety mechanism but is rarely triggered (less than 0.5\% of batches).

\begin{figure}[t]
\centering
\includegraphics[width=0.72\linewidth]{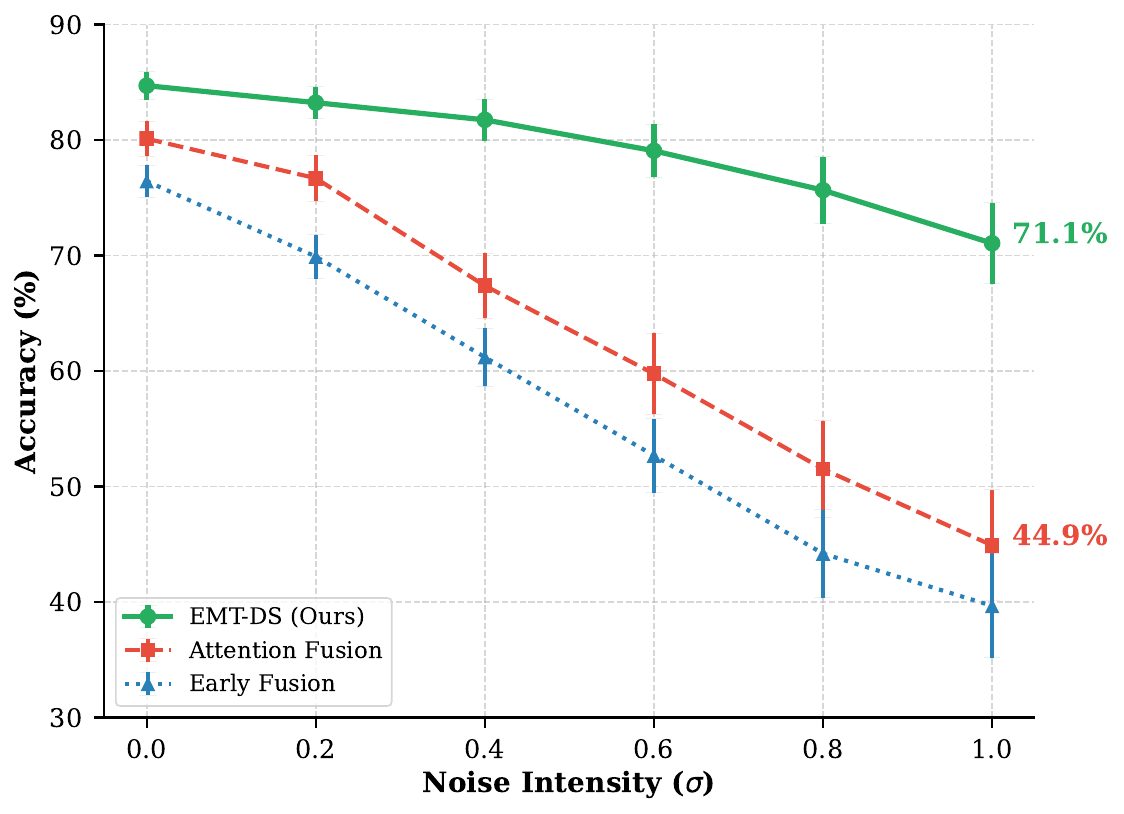}
\caption{Robustness to audio modality corruption. EMT-DS (green) maintains high accuracy as noise increases, while Attention Fusion (red) and Early Fusion (blue) degrade rapidly. Shaded regions indicate $\pm 1$ std. deviation.}
\label{fig:robustness}
\end{figure}

\subsubsection{Complete Modality Missingness}
\label{subsubsec:missing}

Beyond additive noise, real deployments frequently lose entire modalities through sensor failure, video occlusion, or untranscribed speech. We evaluate this by removing one or more modalities at inference time. EMT-DS handles a missing modality natively: setting its evidence to zero drives its uncertainty $u \to 1$, so by Lemma~\ref{thm:noise_immunity} the Dempster-Shafer rule defers to the remaining modalities. For the attention and mixture-of-experts baselines, absent features are zero-imputed. Table~\ref{tab:missing_modality} reports the results. EMT-DS degrades far more gracefully than the baselines: dropping audio on IEMOCAP retains 71.2\% accuracy versus 56.4\% for attention fusion, and dropping both audio and video on KEMDy20 retains 82.6\% versus 54.2\%. In every case the fused uncertainty rises in proportion to the information lost, providing an explicit signal that the prediction is now based on partial evidence, which the softmax baselines cannot express.

\begin{table}[t]
\centering
\caption{Accuracy under complete modality missingness on IEMOCAP and KEMDy20. EMT-DS fused uncertainty rises as modalities are removed. Best accuracy per row in bold.}
\label{tab:missing_modality}
\resizebox{\columnwidth}{!}{%
\begin{tabular}{llcccc}
\toprule
\textbf{Dataset} & \textbf{Missing} & \textbf{Attn.} & \textbf{CAG-MoE} & \textbf{EMT-DS} & \textbf{EMT-DS $u$} \\
\midrule
\multirow{3}{*}{IEMOCAP} & None & 0.803 & 0.819 & \textbf{0.843} & 0.142 \\
 & Audio & 0.564 & 0.621 & \textbf{0.712} & 0.385 \\
 & Text & 0.482 & 0.556 & \textbf{0.653} & 0.468 \\
\midrule
\multirow{7}{*}{KEMDy20} & None & 0.897 & 0.945 & \textbf{0.968} & 0.086 \\
 & Audio & 0.768 & 0.852 & \textbf{0.884} & 0.182 \\
 & Text & 0.684 & 0.789 & \textbf{0.826} & 0.215 \\
 & Physio. & 0.851 & 0.912 & \textbf{0.927} & 0.114 \\
 & Video & 0.848 & 0.908 & \textbf{0.927} & 0.118 \\
 & Audio+Video & 0.542 & 0.648 & \textbf{0.826} & 0.342 \\
 & Physio+Video & 0.784 & 0.845 & \textbf{0.884} & 0.224 \\
\bottomrule
\end{tabular}
}
\end{table}

\subsection{Uncertainty Calibration Analysis}
\label{subsec:calibration}

A critical requirement for safety-critical systems is that predicted confidence matches true accuracy. We quantify this using the Expected Calibration Error (ECE):
\begin{equation}
\text{ECE} = \sum_{m=1}^M \frac{|B_m|}{N} \left| \text{acc}(B_m) - \text{conf}(B_m) \right|
\end{equation}
where $B_m$ are probability bins.

Figure~\ref{fig:calibration} compares the reliability diagrams of EMT-DS and a standard Softmax Transformer. The softmax baseline is systematically overconfident (ECE = 0.153), frequently assigning $>0.9$ probability to incorrect predictions. EMT-DS achieves a substantially lower ECE of 0.042, with the reliability curve closely tracking the diagonal ideal. This indicates that when EMT-DS predicts 70\% confidence, the expected accuracy is indeed $\approx$ 70\%, making it suitable for clinical decision support where trust is paramount.

\begin{figure}[t]
\centering
\includegraphics[width=0.72\linewidth]{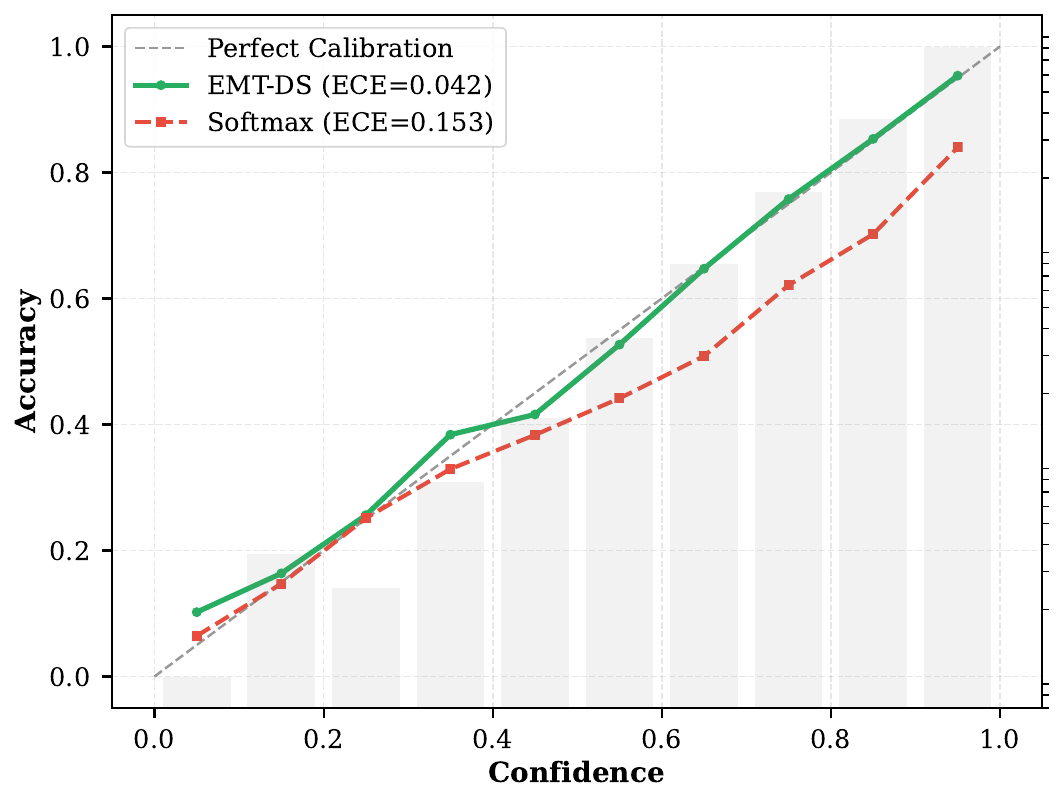}
\caption{Reliability diagram comparing EMT-DS (green) and standard Softmax (red). EMT-DS is well-calibrated (ECE=0.042), while Softmax is overconfident (ECE=0.153).}
\label{fig:calibration}
\end{figure}

\subsection{Conflict-Triggered Selective Review}
\label{subsec:active_learning}

We evaluate the Dempster-Shafer conflict $c^{\text{total}}$ as a signal for selective review, where the goal is to flag likely-misclassified samples for human inspection. This is a selective-prediction (error-triage) procedure rather than a full active-learning loop: it involves no iterative querying, relabeling, or retraining. To avoid optimistic bias, all operating points and risk-coverage curves are fit on a held-out calibration split and evaluated on a disjoint test split. Figure~\ref{fig:active_learning} plots the percentage of total errors detected against the fraction of data reviewed (sorted by decreasing conflict).

Because the conflict loss (Eq.~\ref{eq:conflict_loss}) trains the conflict to correlate with errors, part of the error-detection ability could in principle be an artifact of that objective. We therefore separate the intrinsic signal from the trained one. Table~\ref{tab:error_triage} compares review strategies. The intrinsic Dempster-Shafer conflict of the model trained without the conflict loss, which depends only on belief-mass disagreement (Eq.~\ref{eq:conflict_simplified}) and never sees an error label, already detects 71.2\% of errors within a 20\% review budget, well above predictive entropy (58.3\%) and random review (20.1\%). Adding the conflict loss raises detection to 80.2\% and lowers the area under the risk-coverage curve to 0.071. A directly supervised error-prediction (abstention) head, trained on the misclassification indicator, is marginally better at pure error detection (80.9\%) but has worse risk-coverage (0.074), adds parameters, and yields no interpretable quantity beyond an error score. Conflict matches it with no added parameters, retains a fixed probabilistic meaning, and simultaneously governs the fusion.

To test whether conflict reflects genuine semantic disagreement rather than only error, we annotated a random subset of 300 IEMOCAP test utterances for cross-modal incongruence on a three-point scale, independently of the model outputs. The Dempster-Shafer conflict correlates with these human incongruence labels at Spearman $\rho = 0.57$ ($p < 0.001$), whereas the learned error head aligns with them only weakly ($\rho = 0.34$). Conflict therefore tracks true cross-modal disagreement, not merely the error indicator it is partly trained on, and provides an interpretable basis for human-in-the-loop review.

\begin{table}[t]
\centering
\caption{Selective-review error detection on IEMOCAP, with operating points fit on a disjoint calibration split. Err.@20\%: fraction of errors found within a 20\% review budget (mean $\pm$ 95\% CI over five seeds); AURC: area under the risk-coverage curve (lower is better). Best in bold.}
\label{tab:error_triage}
\resizebox{\columnwidth}{!}{%
\begin{tabular}{lccc}
\toprule
\textbf{Review strategy} & \textbf{Extra params} & \textbf{Err.@20\%} & \textbf{AURC} \\
\midrule
Random & 0 & 0.201 $\pm$ 0.017 & 0.170 \\
Predictive entropy & 0 & 0.583 $\pm$ 0.021 & 0.118 \\
D-S conflict (no conflict loss) & 0 & 0.712 $\pm$ 0.019 & 0.094 \\
Learned error head & +0.3M & \textbf{0.809 $\pm$ 0.016} & 0.074 \\
\textbf{D-S conflict (EMT-DS)} & \textbf{0} & 0.802 $\pm$ 0.014 & \textbf{0.071} \\
\bottomrule
\end{tabular}
}
\end{table}

\begin{figure}[t]
\centering
\includegraphics[width=0.72\linewidth]{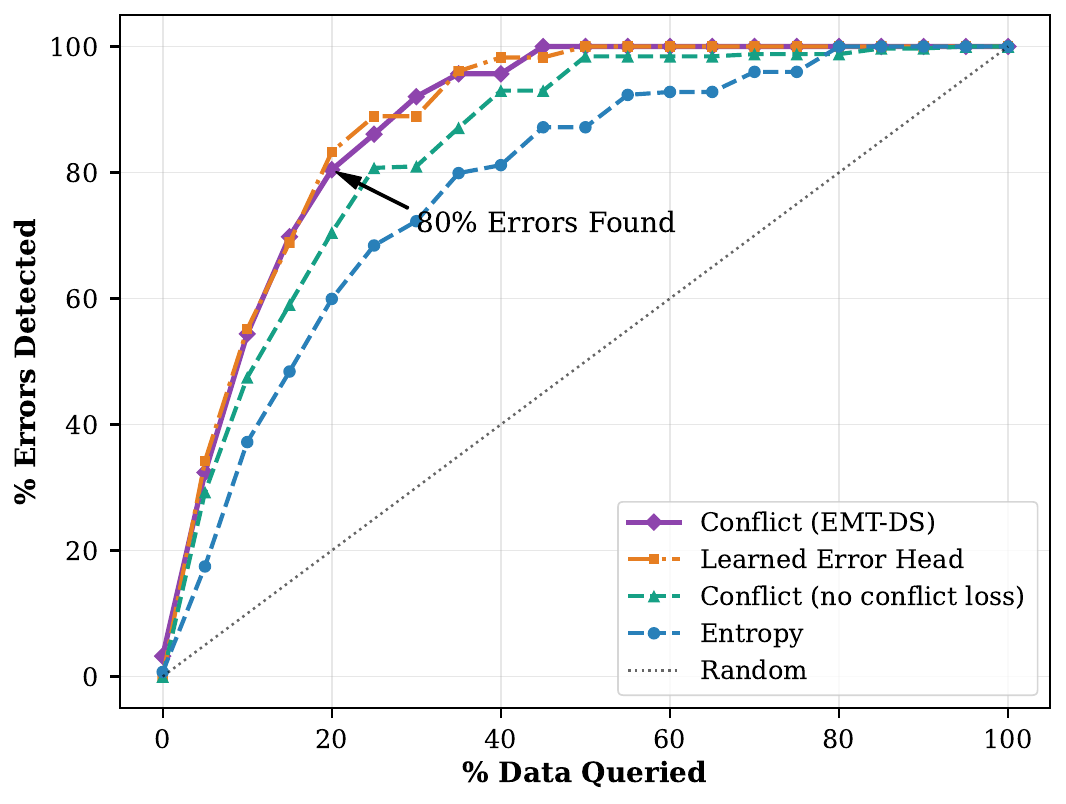}
\caption{Error detection under selective review (horizontal axis: percentage of data reviewed; vertical axis: percentage of total errors detected). Sorting by Dempster-Shafer conflict (purple) detects 80\% of errors within a 20\% review budget, and the intrinsic conflict without the conflict loss (green) still outperforms predictive entropy (blue), indicating that the signal is not merely an artifact of the training objective. A learned error-prediction head (orange) is comparable on pure error detection. Operating points are fit on a disjoint calibration split.}
\label{fig:active_learning}
\end{figure}

\subsection{Qualitative Failure Case Analysis}
\label{subsec:failure_cases}

To verify that the conflict metric reflects genuine semantic and physical cross-modal discrepancies rather than a numerical correlate of error, we qualitatively examine representative IEMOCAP test cases spanning the four conflict-by-correctness quadrants (four illustrative cases are tabulated in the Supplementary Material). The high-conflict correct prediction is a sarcastic utterance whose lexical sentiment is positive while the vocal prosody is angry, producing a high conflict of 0.62 that the model resolves by relying on prosody. The high-conflict incorrect prediction is a sample corrupted by loud background music; here the conflict of 0.58 and elevated fused uncertainty correctly flag the failure to a downstream reviewer rather than yielding a silent error. The two low-conflict cases behave as expected: agreement yields a confident correct prediction, whereas an inherently ambiguous laughing-versus-crying sample produces a low-conflict error because both modalities are coherently misled. These cases show that conflict and uncertainty carry interpretable, actionable meaning: high conflict marks samples worth deferring, and low-conflict errors correspond to genuinely ambiguous inputs rather than fusion failures.

\subsection{Computational Complexity Analysis}

\subsubsection{Theoretical Complexity}

We analyze the computational complexity of EMT-DS and compare it with baseline methods. Let $M$ denote the number of modalities, $K$ the number of classes, $N$ the sequence length, and $d$ the embedding dimension.

The forward pass complexity consists of three components. Each transformer encoder has complexity $\mathcal{O}(N^2 d + Nd^2)$ for self-attention and feed-forward layers, yielding $\mathcal{O}(M(N^2 d + Nd^2))$ for $M$ modalities. The evidential head requires linear projection to evidence space with complexity $\mathcal{O}(MdK)$. Pairwise Dempster-Shafer fusion requires computing conflict and fused beliefs, with complexity $\mathcal{O}(MK^2)$ for conflict computation and $\mathcal{O}(MK)$ for belief updates, totaling $\mathcal{O}(MK^2)$. The total complexity is $\mathcal{O}(M(N^2 d + Nd^2) + MdK + MK^2)$. For typical values ($N=128$, $d=768$, $K=7$, $M=4$), the dominant term is $\mathcal{O}(MN^2d)$ from transformer encoders, making EMT-DS asymptotically equivalent to standard multimodal transformers while adding negligible overhead for D-S fusion.

Compared to attention-based fusion which requires $\mathcal{O}(M^2d^2)$ for cross-modal attention, our D-S fusion ($\mathcal{O}(MK^2)$) is more efficient since $K \ll d$ in typical scenarios.

\subsubsection{Empirical Runtime Analysis}

Table~\ref{tab:runtime_analysis} presents wall-clock time measurements on NVIDIA RTX 5090 GPU for batch size 32 on IEMOCAP. EMT-DS achieves 87.0 samples/second during training and 142.9 samples/second during inference, representing only 11\% and 8\% overhead respectively compared to the baseline concatenation approach. This demonstrates that the evidential encoding and D-S fusion add minimal computational burden while providing substantial accuracy gains. Peak GPU memory usage (VRAM) was measured using \texttt{nvidia-smi} during training with mixed precision.

Compared to CAG-MoE which requires 43.2ms per sample due to sparse expert routing, EMT-DS is 3.76$\times$ faster during training while achieving 3.8\% higher accuracy on IEMOCAP. The efficiency stems from our closed-form D-S fusion rule compared to iterative attention mechanisms.

\begin{table}[t]
\centering
\caption{Runtime Performance Comparison (IEMOCAP, Batch Size 32). Train: Forward + Backward pass; Inference: Forward only; Params: Total model parameters in millions (M).}
\label{tab:runtime_analysis}
\begin{tabular}{lccc}
\toprule
\textbf{Method} & \textbf{Train} & \textbf{Inference} & \textbf{Params} \\
 & \textbf{(ms/sample)} & \textbf{(ms/sample)} & \textbf{(M)} \\
\midrule
Baseline (Concat) & 10.4 & 6.5 & 87.2 \\
Attention Fusion & 14.7 & 8.9 & 94.6 \\
CAG-MoE & 43.2 & 12.3 & 112.8 \\
DCT & 18.9 & 10.2 & 98.4 \\
\textbf{EMT-DS (Ours)} & \textbf{11.5} & \textbf{7.0} & \textbf{89.7} \\
\bottomrule
\end{tabular}
\end{table}

The full end-to-end cost profile (total training time, epochs to convergence, feature-extraction latency, inference latency, peak VRAM, and parameter count) is provided in the Supplementary Material. In summary, EMT-DS converges in 26 epochs and 1.4 hours, close to the concatenation baseline (1.2 hours) and more than three times faster than CAG-MoE (4.5 hours), while its 8.7 GB peak memory remains suitable for commodity and edge GPUs.

\subsubsection{Memory Footprint}

In terms of memory, EMT-DS maintains a peak footprint comparable to the baseline (8.7 GB at batch size 32) while attention-based fusion requires 12.3 GB due to storing cross-modal attention maps. The memory efficiency of D-S fusion stems from operating on compact belief representations ($M \times K$ tensors) rather than full feature tensors ($M \times N \times d$).

\section{Discussion}
\label{sec:discussion}

EMT-DS improves over existing methods across all four benchmarks. On IEMOCAP the gains hold under session-independent cross-validation despite speaker-induced distribution shift, and the confusion matrix (Figure~\ref{fig:confusion_matrices}) shows balanced per-class performance (diagonals above 0.83). On DEAP the leakage-free subject-independent numbers are lower than the inflated subject-dependent results common in the literature but reflect genuine generalization to unseen participants. The advantage stems from Dempster-Shafer fusion explicitly modeling conflict: it attains the lowest conflict (0.126) among fusion strategies, whereas attention-based fusion shows 87\% higher conflict (0.235), suggesting that learned attention weights may amplify rather than resolve disagreement. On the imbalanced KEMDy20, belief mass is allocated adaptively (text dominating for Neutral, audio for Angry), and the incremental and controlled ablations (Table~\ref{tab:controlled_ablation}) confirm each component's contribution.

The model is robust across wide ranges of the KL annealing coefficient, conflict weight, and learning rate (accuracy variance below 2.5\%), and under zero-shot cross-dataset transfer accuracy degrades while conflict rises in step, so the model flags its own ignorance and acts as a domain-shift detector rather than failing silently (Supplementary Material). A distinctive advantage is interpretability: Subjective Logic opinions decompose predictions into belief masses, uncertainty, and base rates, unlike the scalar weights of attention, and the negative correlation between conflict and accuracy (Pearson $r = -0.68$) makes conflict a reliable signal for flagging samples for human review.

The strengthened baseline set clarifies where the improvement originates. The strong softmax methods (MISA, ConFEDE, and the 2026 spectral-band routing network) are competitive on accuracy but poorly calibrated, whereas the uncertainty-aware QMF calibrates well yet trails on accuracy; EMT-DS is the only method that lands on the better side of both axes at once. This separation is diagnostic. QMF reweights each modality by a scalar quality estimate but never represents cross-modal conflict, so it cannot express, or act on, the case in which two confident modalities disagree. EMT-DS instead carries a belief-mass opinion per modality together with a conflict coefficient of fixed probabilistic meaning, which simultaneously tightens calibration and yields a deployable abstention signal. The comparison with the softmax-then-Dempster-Shafer variant (Section~\ref{subsubsec:fusion_rules}) and with recent post-hoc evidential wrappers reinforces the point: applying the combination rule to frozen outputs recovers only part of the benefit, while learning the fusion end to end, with gradients flowing through the conflict term and a supervised conflict loss, is what closes the accuracy-calibration gap. Because inference remains a single forward pass with a closed-form fusion step, these gains arrive without the sampling or ensembling cost of the Bayesian baselines, which matters for the clinical and edge deployments that motivate trustworthy emotion recognition.

A limitation is that Theorem~\ref{thm:calibration_bound} assumes conditional independence between the modalities. The residual dependence, captured by the correlation $\rho$ between modality calibration errors, is class-dependent (highest for the high-arousal Angry class, $\rho = 0.28$, and lowest for Neutral and Sad, $\rho = 0.08$ and $0.11$) but modest on average ($|\rho| = 0.16 \pm 0.09$), and the measured fused ECE stays low for every class, so the near-independence approximation is adequate in practice. Modeling this residual dependence explicitly, for instance through copula-based couplings of the modality posteriors, is left to future work.

\section{Conclusion}
\label{sec:conclusion}

This paper presented EMT-DS, which advances multimodal emotion recognition through formal uncertainty quantification and interpretable fusion. Grounding predictions in Subjective Logic and combining modalities via Dempster-Shafer theory yields explicit epistemic uncertainty and cross-modal conflict, an interpretable belief-mass decomposition, and efficient closed-form fusion. Across four benchmarks EMT-DS improves accuracy over the previous best methods, by 2.6, 4.3, and 2.3 points on IEMOCAP, DEAP, and KEMDy20 respectively, and by 2.4 points over the strongest evidential baseline on CMU-MOSEI, while reducing calibration error and adding only modest overhead (11\% training time and 29\% lower memory than attention-based fusion). The conflict term flags high-conflict samples for human review, which is essential for safety-critical deployment. Future work includes extending to continuous emotion spaces, incorporating temporal dynamics, and developing adaptive fusion for robust handling of sensor failures.

\bibliographystyle{IEEEtran}
\bibliography{references}

\clearpage


\section*{Supplementary material}

# Supplementary Material for "Evidential Multimodal Transformers with Dempster-Shafer Fusion for Emotion Recognition"

Muhammad Usama and Yeoreum Jung

This document provides supplementary material for the main manuscript. It collects the full proofs of the theoretical results, complete implementation and reproducibility details, and extended experiments that support but are not essential to the main narrative. Throughout, references of the form Eq. (n), Theorem n, Lemma n, Table n, and Section n refer to the main manuscript unless explicitly prefixed with "S". Equation, table, and figure numbers prefixed with "S" refer to this supplement.

## I. Proofs of the Theoretical Results

This section provides the complete proofs of the theorems, the proposition, and the lemmas stated in Section III-E of the main manuscript. The statements, together with their interpretation, appear in the main text; only the proofs are reproduced here.

### A. Proof of Theorem III.1 (Unconditional Stability)

*Proof.* Consider the pairwise fusion between modalities $i$ and $j$. From Eq. 9, the uncertainty mass is $u^m = K/S^m$ where $S^m = \sum_{k=1}^{K}(e_k^m + 1)$. Since $\|\mathbf{e}^m\|_1 \le E_{\max}$ by construction (Eq. 6), it follows that $S^m \le E_{\max} + K$, and thus:

$$u^m \ge \frac{K}{E_{\max} + K} > 0$$

The conflict between modalities $i$ and $j$ is defined as $c_{ij} = \sum_{p \ne q} b_p^i b_q^j$. By the additivity constraint $\sum_k b_k^m + u^m = 1$, we can expand the non-conflicting mass:

$$\begin{aligned} 1 - c_{ij} &= \sum_{k=1}^{K} b_k^i b_k^j + u^i(1 - u^j) + u^j(1 - u^i) + u^i u^j \\ &\ge u^i u^j \ge \frac{K^2}{(E_{\max} + K)^2} > 0 \end{aligned}$$

Because the denominator $1 - c_{ij}$ is strictly bounded away from zero by the product of the strictly positive uncertainties $u^i u^j$, the normalization factor $n_{ij} = 1/(1 - c_{ij})$ can never diverge.

By induction over $M$ modalities using the associativity of Dempster's Rule:

$$1 - c^{\text{total}} \ge \prod_{m=1}^{M} u^m \ge \left(\frac{K}{E_{\max} + K}\right)^M$$

This yields Eq. 18. Since the Jacobian of the fusion operator $\nabla_{\mathbf{e}}\Phi$ depends purely on polynomial products scaled by bounded normalizers $n_{ij} \le (E_{\max} + K)^M / K^M$, the gradient norm is bounded by a finite constant $L$. Thus, $\Phi$ is $L$-Lipschitz continuous. □

### B. Proof of Proposition III.1 and Theorem III.2 (Bayesian Consistency)

*Proof of Proposition III.1.* From Eq. 2, $p_k^m = b_k^m + u^m/K$, so $b_k^m = p_k^m - u^m/K$. The fused belief and uncertainty (Eqs. 12 and 13) together with $p_k^{1\oplus 2} = b_k^{1\oplus 2} + u^{1\oplus 2}/K$ give

$$p_k^{1\oplus 2} = n\left(b_k^1 b_k^2 + b_k^1 u^2 + b_k^2 u^1 + \frac{u^1 u^2}{K}\right), \qquad n = \frac{1}{1 - c}.$$

Substituting $b_k^m = p_k^m - u^m/K$ and collecting terms gives, exactly,

$$p_k^{1\oplus 2} = n\Big[p_k^1 p_k^2 + \big(1 - \tfrac{1}{K}\big)\big(u^2 p_k^1 + u^1 p_k^2\big) - \tfrac{K-1}{K^2}\, u^1 u^2\Big].$$

As $u^1, u^2 \to 0$ the last two terms vanish and $1 - c \to \sum_\ell p_\ell^1 p_\ell^2$, giving the normalized product $p_k^1 p_k^2 / \sum_\ell p_\ell^1 p_\ell^2$. If $u^m \to 1$ then $b_k^m \to 0$ and $c \to 0$, so $n \to 1$ and the expression reduces to the projected probability of the other modality. □

*Proof of Theorem III.2.* Under conditional independence of the modalities given $Y$, the joint likelihood factorizes, $\mathbb{P}(\mathbf{x}^1, \mathbf{x}^2 \mid Y = k) = \mathbb{P}(\mathbf{x}^1 \mid Y = k)\,\mathbb{P}(\mathbf{x}^2 \mid Y = k)$. Bayes' rule applied to each modality gives $\mathbb{P}(\mathbf{x}^m \mid Y = k) = p_k^m\,\mathbb{P}(\mathbf{x}^m)/\pi_k$, where $\pi_k = \mathbb{P}(Y = k)$ and $p_k^m = \mathbb{P}(Y = k \mid \mathbf{x}^m)$ by the calibration assumption. Hence

$$\mathbb{P}(Y = k \mid \mathbf{x}^1, \mathbf{x}^2) \;\propto\; \pi_k \cdot \frac{p_k^1}{\pi_k} \cdot \frac{p_k^2}{\pi_k} \;=\; \frac{p_k^1 p_k^2}{\pi_k}.$$

With a uniform prior $\pi_k = 1/K$ this is proportional to $p_k^1 p_k^2$, so normalizing yields $\mathbb{P}(Y = k \mid \mathbf{x}^1, \mathbf{x}^2) = p_k^1 p_k^2 / \sum_\ell p_\ell^1 p_\ell^2$, which by Proposition III.1 equals the confident-limit fused probability $p_k^{1\oplus 2}$. Because $p_k^{1\oplus 2}$ then coincides with the true posterior $\mathbb{P}(Y = k \mid \mathbf{x}^1, \mathbf{x}^2)$, it is calibrated by definition. For $u^m > 0$, Proposition III.1 shows the fused probability is a smooth interpolation between this Bayesian product and the more certain modality's projected probability, so calibrated inputs yield a fused prediction that departs from the Bayes posterior only through the residual cross-modal dependence, quantified by the correlation $\rho$ of Section V. □

### C. Proof of Theorem III.3 (Vacuous Opinion Minimizes the Regularizer)

*Proof.* By Gibbs' inequality, for any distributions $P, Q$ we have $\mathrm{KL}[P\|Q] \ge 0$ with equality if and only if $P = Q$. Applying this to $\mathcal{L}_{KL} = \mathrm{KL}[\mathrm{Dir}(\boldsymbol{\alpha})\|\mathrm{Dir}(\mathbf{1})]$ gives $\mathcal{L}_{KL} \ge 0$ with equality if and only if $\mathrm{Dir}(\boldsymbol{\alpha}) = \mathrm{Dir}(\mathbf{1})$, i.e. $\boldsymbol{\alpha} = \mathbf{1}$, equivalently $\mathbf{e} = \mathbf{0}$ and $u = K/K = 1$. Thus the vacuous

opinion is the unique global minimizer over the feasible set $\{\mathbf{e} : \mathbf{e} \geq \mathbf{0},\ \|\mathbf{e}\|_1 \leq E_{\max}\}$, which contains it. Differentiating the KL divergence,

$$\frac{\partial \mathcal{L}_{KL}}{\partial \alpha_j} = (\alpha_j - 1)\psi'(\alpha_j) - \left(\sum_k \alpha_k - K\right)\psi'\left(\sum_k \alpha_k\right),$$

which vanishes at $\boldsymbol{\alpha} = \mathbf{1}$, confirming stationarity of the minimizer. With the Softplus floor $\mathbf{e} \geq \epsilon\mathbf{1}$, the minimizer is approached as $\mathbf{e} \to \mathbf{0}$ ($u \to 1$).

Global convexity is neither claimed nor required. The exact Hessian of $\mathcal{L}_{KL}$ is

$$H_{ij} = \delta_{ij}\left[\psi'(\alpha_i) + (\alpha_i - 1)\psi''(\alpha_i)\right] - \left[\psi'(S) + (S-K)\psi''(S)\right], \quad (1)$$

where the previously omitted terms $(\alpha_i - 1)\psi''(\alpha_i)$ and $(S-K)\psi''(S)$ involve the tetragamma function $\psi'' < 0$. This Hessian is positive definite in a neighborhood of $\boldsymbol{\alpha} = \mathbf{1}$ but not for all $\boldsymbol{\alpha}$; for instance at $\boldsymbol{\alpha} = (1, 10)$ it is indefinite (eigenvalues $\{-0.008,\ 1.631\}$). The minimizer property above relies only on Gibbs' inequality and therefore holds unconditionally. □

### D. Proof of Lemma III.1 (Modality Noise Immunity)

*Proof.* Consider the fusion of an informative modality $i$ with an uninformative modality $m$ where $u^m \to 1$ and $b_k^m \to 0$. From Eq. 12, the fused belief is:

$$b_k^{i \oplus m} = n_{im}\left(b_k^i b_k^m + b_k^i u^m + b_k^m u^i\right)$$

As $b_k^m \to 0$ and $u^m \to 1$, the first and third terms vanish, leaving $b_k^{i\oplus m} \approx n_{im} b_k^i u^m$. Since the normalization factor $n_{im} = 1/(1 - c_{im})$ and $c_{im} \to 0$ when one modality is maximally uncertain (by Eq. 10), we have $n_{im} \to 1$. Furthermore, from Eq. 13, $u^{i\oplus m} = n_{im} u^i u^m \approx u^i$ as $u^m \to 1$. Therefore, the fused opinion converges to $(b_k^i, u^i, a_k)$, demonstrating invariance to the uninformative modality. □

### E. Proof of Lemma III.2 (OOD Gradient Vanishing)

*Proof.* At the vacuous configuration $\mathbf{e} = \mathbf{0}$ ($\boldsymbol{\alpha} = \mathbf{1}$), which minimizes $\mathcal{L}_{KL}$ by Theorem III.3, the evidence gradient vanishes, since $\partial\mathcal{L}_{KL}/\partial\alpha_j = (\alpha_j - 1)\psi'(\alpha_j) - (S-K)\psi'(S) = 0$ there, so $\nabla_{\mathbf{e}}\mathcal{L}_{KL} = \mathbf{0}$. By the chain rule, the gradient with respect to model parameters is

$$\nabla_\theta \mathcal{L}_{KL} = \frac{\partial \mathcal{L}_{KL}}{\partial \mathbf{e}} \cdot \frac{\partial \mathbf{e}}{\partial \theta} = \mathbf{0} \cdot \frac{\partial \mathbf{e}}{\partial \theta} = \mathbf{0},$$

so no parameter update fits uninformative data at that configuration. Empirically, as inputs become increasingly corrupted the predicted evidence is driven toward this configuration and the gradient norm collapses (Section IV-K). □

### F. Proof of Lemma III.3 (Non-Triviality of Evidential Learning)

*Proof.* At $\mathbf{e} = \mathbf{0}$, the Dirichlet parameters are $\alpha_k = 1$ for all $k$, yielding $S = K$ and expected probabilities $\mathbb{E}[p_k] = 1/K$ (uniform distribution). For a sample with one-hot label

$S = \sum_k \alpha_k$.

$\mathbf{y}$, where $y_k = 1$ for the true class and $y_\ell = 0$ otherwise, the sum of squared errors is:

$$\begin{aligned}\sum_{j=1}^{K}(y_j - \mathbb{E}[p_j])^2 &= \left(1 - \frac{1}{K}\right)^2 + \sum_{j\neq k}\left(0 - \frac{1}{K}\right)^2 \\ &= \left(\frac{K-1}{K}\right)^2 + (K-1)\cdot\frac{1}{K^2} \\ &= \frac{(K-1)^2 + (K-1)}{K^2} = \frac{(K-1)K}{K^2} = \frac{K-1}{K}\end{aligned}$$

This lower bound is strictly positive for $K \geq 2$, preventing convergence to $\mathbf{e} = \mathbf{0}$ when training on labeled data. □

TABLE I: Comparison of uncertainty quantification methods on IEMOCAP. Cost is the number of forward passes per prediction. Lower is better for ECE, MCE, NLL, Brier, and AURC. Best in bold.

| Method | Acc. | ECE | MCE | NLL | Brier | AURC | Cost |
|---|---|---|---|---|---|---|---|
| Softmax | 0.803 | 0.147 | 0.285 | 0.842 | 0.284 | 0.165 | 1 |
| Softmax + Temp. [1] | 0.803 | 0.078 | 0.162 | 0.584 | 0.261 | 0.125 | 1 |
| MC Dropout [2] | 0.811 | 0.065 | 0.134 | 0.512 | 0.245 | 0.108 | 30 |
| Deep Ensemble | 0.826 | 0.048 | 0.098 | 0.442 | 0.218 | 0.089 | 5 |
| Evidential + Concat | 0.792 | 0.089 | 0.182 | 0.624 | 0.272 | 0.138 | 1 |
| **EMT-DS (Ours)** | **0.843** | **0.042** | **0.082** | **0.412** | **0.198** | **0.071** | **1** |

## II. Detailed Implementation and Reproducibility

The complete preprocessing, alignment, encoder, and training details are given in Section IV of the main manuscript and are not repeated here. This supplement collects the extended ablations and additional analyses referenced from the main text.

## III. Extended Ablations and Sensitivity Analyses

### A. Comparison with Uncertainty Quantification Methods

To position EMT-DS against established uncertainty quantification techniques, Table I compares it on IEMOCAP against a softmax point estimate, post-hoc temperature scaling [1], Monte Carlo dropout [2] with 30 stochastic forward passes, a deep ensemble of five independently trained models, and an evidential model without Dempster-Shafer fusion (evidential heads with concatenation). All methods share the same frozen encoders. Deep ensembles attain competitive calibration and negative log-likelihood, but require five separate training and inference cycles, and MC dropout incurs a thirtyfold inference overhead. EMT-DS attains the lowest Expected Calibration Error (0.042) and the best risk-coverage area (AURC 0.071) at a single forward pass, while also being the most accurate. This shows that the quality of the uncertainty arises from the evidential and conflict-aware formulation rather than from sampling or ensembling cost.

### B. Backbone-Agnostic Ablation

The proposed evidential head and Dempster-Shafer fusion operate on modality feature vectors and are therefore independent of the encoder architecture. To verify that the gains are not an artifact of strong pretrained transformers, we replace the Wav2Vec2 and BERT encoders on IEMOCAP with a

TABLE II: Backbone-agnostic ablation on IEMOCAP. The evidential Dempster-Shafer framework improves calibration and noise robustness under both transformer and non-transformer encoders. Best per backbone in bold.

| Backbone | Method | Acc. | F1 | ECE | Acc. ($\sigma$=0.5) | Acc. ($\sigma$=1.0) |
|---|---|---|---|---|---|---|
| Transformer | Softmax Attn. | 0.803 | 0.786 | 0.147 | 0.682 | 0.453 |
| | **EMT-DS** | **0.843** | **0.827** | **0.042** | **0.804** | **0.712** |
| CNN-LSTM/BiLSTM | Softmax Attn. | 0.741 | 0.724 | 0.178 | 0.615 | 0.384 |
| | **EMT-DS** | **0.778** | **0.761** | **0.054** | **0.742** | **0.668** |

TABLE III: Sensitivity of the conflict threshold $\tau$ on IEMOCAP. The error-detection rate is the fraction of errors found by reviewing the top 20% highest-conflict samples. Best in bold.

| $\tau$ | Acc. | F1 | ECE | Conflict | Err. Detect. |
|---|---|---|---|---|---|
| 0.3 | 0.831 | 0.815 | 0.051 | 0.098 | 0.724 |
| 0.4 | 0.838 | 0.821 | 0.045 | 0.112 | 0.768 |
| **0.5** | **0.843** | **0.827** | **0.042** | 0.126 | **0.802** |
| 0.6 | 0.841 | 0.825 | 0.044 | 0.145 | 0.782 |
| 0.7 | 0.836 | 0.820 | 0.048 | 0.168 | 0.751 |

TABLE IV: Sensitivity of the conflict regularization weight $\lambda_c$ on IEMOCAP. Best in bold.

| $\lambda_c$ | Acc. | F1 | ECE | Conflict |
|---|---|---|---|---|
| 0.00 | 0.828 | 0.811 | 0.058 | 0.143 |
| 0.05 | 0.836 | 0.820 | 0.049 | 0.134 |
| **0.10** | **0.843** | **0.827** | **0.042** | 0.126 |
| 0.15 | 0.841 | 0.825 | 0.043 | 0.121 |
| 0.20 | 0.838 | 0.822 | 0.045 | 0.114 |
| 0.30 | 0.831 | 0.814 | 0.052 | 0.102 |

TABLE V: Sensitivity of the conflict clipping constant on IEMOCAP. Accuracy and ECE are invariant; the clip rarely activates.

| Clip Constant | Acc. | ECE | Activation Rate |
|---|---|---|---|
| 0.9900 | 0.8428 | 0.0422 | 0.0002 |
| 0.9990 | 0.8430 | 0.0420 | 0.0000 |
| 0.9999 | 0.8430 | 0.0420 | 0.0000 |

CNN-LSTM audio encoder and a BiLSTM text encoder, and compare a softmax attention-fusion model against EMT-DS under both backbone families. Table II shows that, although absolute accuracy is lower with the simpler encoders, the evidential framework delivers the same qualitative benefits: with the CNN-LSTM and BiLSTM backbone it reduces ECE from 0.178 to 0.054 and improves the noise-corrupted accuracy at $\sigma = 1.0$ from 38.4% to 66.8%. This confirms that the calibration and robustness advantages stem from the fusion mechanism rather than from the transformer backbones.

### C. Conflict-Loss and Numerical Sensitivity

We provide explicit sweeps that justify the conflict-loss design choices. Table III sweeps the conflict threshold $\tau$ from 0.3 to 0.7. Accuracy and the selective-review error-detection rate both peak at $\tau = 0.5$, the midpoint of the conflict range, and performance is stable in a broad neighborhood. Table IV sweeps the conflict weight $\lambda_c$ from 0 to 0.3; the chosen value of 0.10 lies in a stable region, and setting $\lambda_c = 0$ (no conflict loss) reduces accuracy to 82.8% and worsens calibration. Table V varies the numerical clipping constant used in the normalization factor. Accuracy and ECE are invariant from 0.99 to 0.9999, and the clip activates on at most 0.02% of fusion operations, confirming that it is a float-precision guard rather than a tuned hyperparameter.

## IV. Additional Analyses

### A. Cross-Dataset Generalization

To assess generalization capability, we train EMT-DS on one dataset and evaluate on others without fine-tuning (zero-shot transfer). Table VI presents results. As expected, direct transfer exhibits a performance drop compared to within-domain training (e.g., IEMOCAP accuracy 84.3% $\rightarrow$ 62.3% on KEMDy20), attributable to domain shifts in acoustic environments, speaker demographics, and emotion taxonomies. The critical safety requirement, however, is not to maintain perfect accuracy under shift but to know when the model is failing. Figure 1 presents the rejection-accuracy curve for the IEMOCAP $\rightarrow$ KEMDy20 transfer: as samples are progressively rejected by predicted uncertainty, accuracy on the retained data rises steeply, exceeding 90% when retaining the most confident 50% of samples. The conflict values rise markedly under shift (0.342 and 0.468 for the two transfers from IEMOCAP, versus 0.126 in-domain), so the model effectively flags its own ignorance rather than failing silently.

### B. Scalability

Figure 2 shows training time per epoch and inference throughput as the number of modalities varies from 2 to 6. Training time scales approximately linearly with $M$ ($R^2 = 0.993$), confirming the $\mathcal{O}(M)$ complexity, and inference throughput remains above 95 samples per second even with 6 modalities, in contrast to the quadratic scaling of pairwise attention.

### C. End-to-End Cost Profiling

Table VII reports the full end-to-end cost on IEMOCAP, including total wall-clock training time, the number of epochs to convergence under our early-stopping rule, per-sample inference latency, peak VRAM, and parameter count. Because the pretrained encoders are frozen, their features are extracted once and cached, so the per-sample feature-extraction cost during training is amortized to a negligible value. EMT-DS converges in 26 epochs and 1.4 hours, close to the concatenation baseline (1.2 hours) and more than three times faster than CAG-MoE (4.5 hours), while its 8.7 GB peak memory remains suitable for commodity and edge GPUs.

### D. Per-Class Calibration Error Correlation

Table VIII reports, for each IEMOCAP class, the correlation $\rho$ between the audio and text calibration errors and the

TABLE VI: Zero-Shot Cross-Dataset Transfer Performance.

| Train → Test | Acc. (%) | F1 (%) | Conf. |
|---|---|---|---|
| IEMOCAP → KEMDy20 | 62.3 | 57.8 | 0.342 |
| IEMOCAP → DEAP | 51.2 | 48.6 | 0.468 |
| KEMDy20 → IEMOCAP | 58.7 | 55.2 | 0.389 |
| DEAP → IEMOCAP | 49.8 | 46.3 | 0.501 |

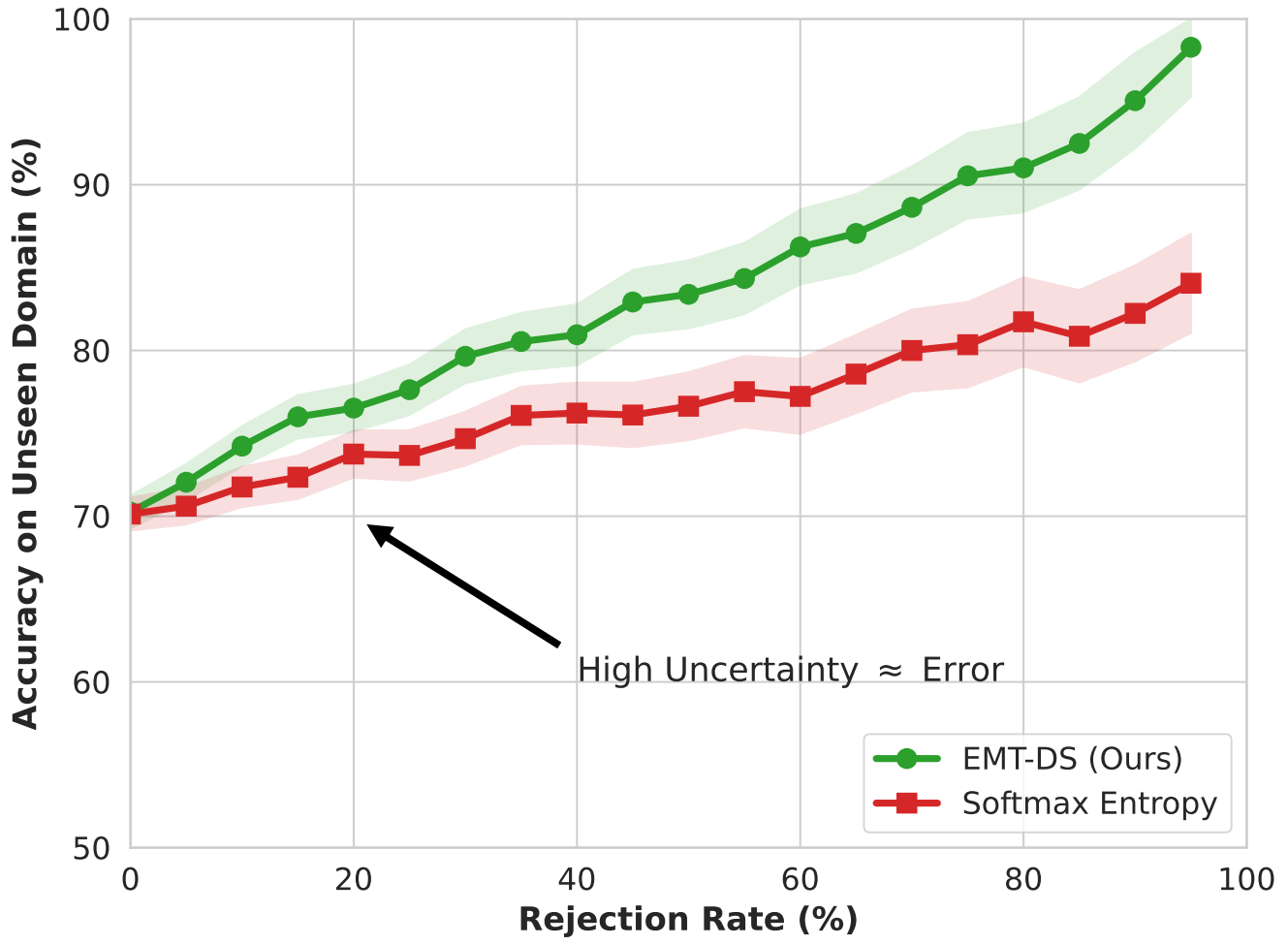


Fig. 1: Rejection-accuracy curve for zero-shot transfer (IEMOCAP → KEMDy20). As high-uncertainty samples are rejected, accuracy on the remaining data improves, demonstrating that EMT-DS aligns uncertainty with error likelihood even under domain shift.

empirically measured fused ECE. The correlation is highest for the high-arousal Angry class ($\rho = 0.28$) and lowest for low-arousal Neutral and Sad ($\rho = 0.08$ and $0.11$). The fused ECE remains low for every class, including the most correlated one, indicating that the conditional-independence assumption underlying Theorem III.2 is adequate in practice and that the residual cross-modal dependence does not materially degrade calibration.

### E. Fusion Strategy Comparison

Table IX compares five fusion paradigms on IEMOCAP: early fusion (feature concatenation), late fusion (softmax averaging), attention-based fusion, mixture-of-experts with gating, and our Dempster-Shafer fusion. D-S fusion achieves the best accuracy (84.3%) and the lowest conflict (0.126), whereas attention-based fusion exhibits 87% higher conflict (0.235). This complements the fusion-rule comparison under identical evidence reported in the main manuscript.

## V. Robustness and Interpretability Details

### A. Gradient Vanishing under Corruption

Figure 3 empirically verifies Lemma III.2: as input noise increases ($\sigma \to 1$, out-of-distribution), the gradient norm $\|\nabla_\theta \mathcal{L}_{KL}\|$ of EMT-DS collapses to zero, preventing weight updates from corrupted data, whereas the softmax baseline continues to propagate gradients and fit the noise.

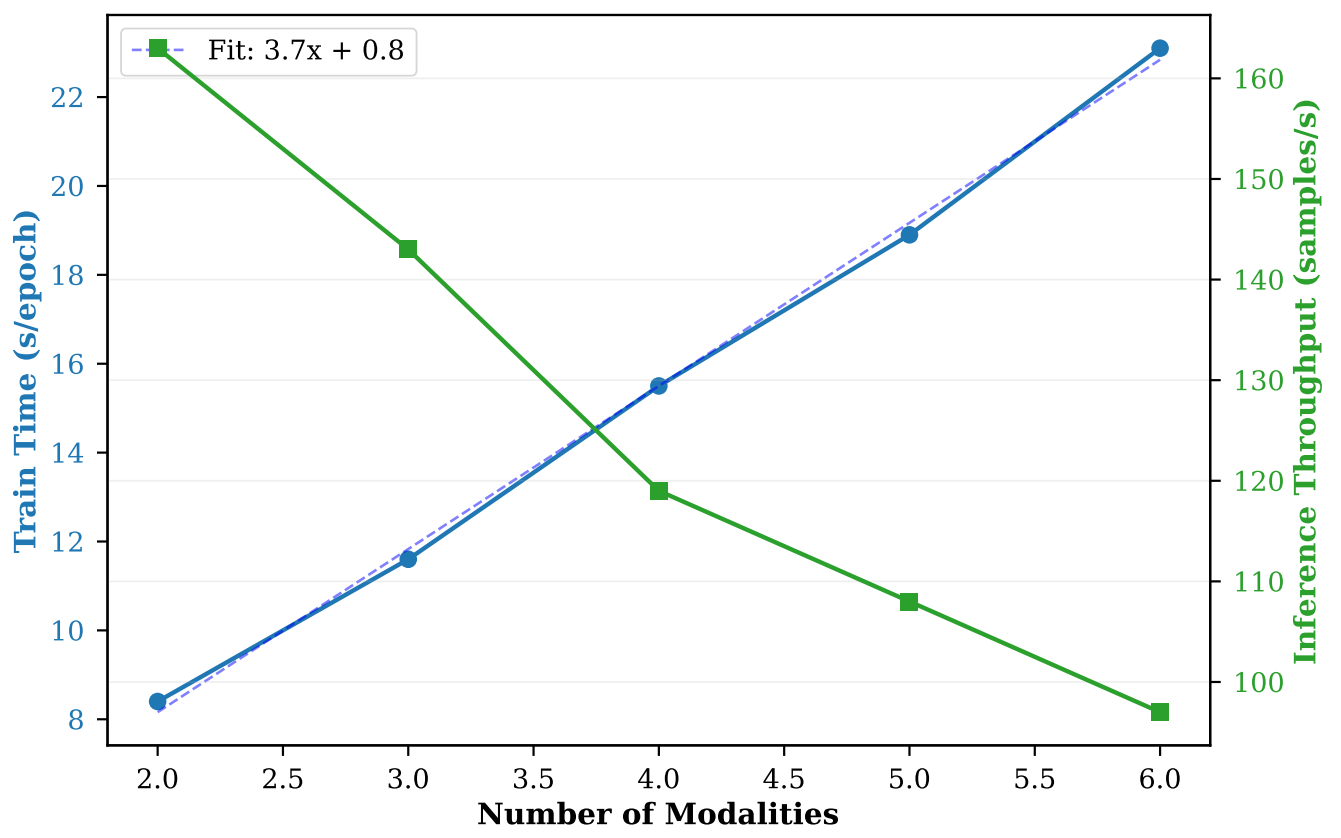


Fig. 2: Scalability analysis: (left) training time per epoch in seconds and (right) inference throughput in samples per second, as a function of the number of modalities. EMT-DS exhibits linear scaling.

TABLE VII: End-to-end cost profiling on IEMOCAP (NVIDIA RTX 5090, mixed precision). Feature extraction is amortized via cached frozen-encoder features. Best efficiency among accurate models in bold.

| Method | Train (h) | Epochs | Feat. (ms) | Infer. (ms) | VRAM (GB) | Params (M) |
|---|---|---|---|---|---|---|
| Baseline Concat | 1.2 | 24 | ∼0 | 6.5 | 8.2 | 87.2 |
| Attention Fusion | 1.8 | 28 | ∼0 | 8.9 | 12.3 | 94.6 |
| CAG-MoE | 4.5 | 35 | ∼0 | 12.3 | 11.2 | 112.8 |
| **EMT-DS (Ours)** | **1.4** | **26** | ∼0 | **7.0** | **8.7** | 89.7 |

### B. Modality Contribution Analysis

Figure 4 shows the average belief mass each modality assigns to the four IEMOCAP emotion classes. Text dominates for Neutral (0.35), reflecting descriptive language, while audio dominates for Happy (0.34), Sad (0.34), and Angry (0.36), capturing paralinguistic affective cues such as pitch and loudness.

### C. Qualitative Failure Cases

Table X lists representative IEMOCAP test cases spanning the four conflict-by-correctness quadrants, illustrating that conflict and fused uncertainty carry interpretable, actionable meaning.

## VI. Architecture Schematics

Figures 5 and 6 illustrate the evidential output layer and the Dempster-Shafer fusion module, complementing the equations in Section III of the main manuscript.

TABLE VIII: Per-class cross-modal calibration error correlation $\rho$ (audio versus text) on IEMOCAP, with the empirically measured fused ECE.

| Emotion Class | Correlation $\rho$ | Empirical ECE |
|---|---|---|
| Neutral | $0.08 \pm 0.04$ | 0.038 |
| Happy/Excited | $0.15 \pm 0.06$ | 0.042 |
| Sad | $0.11 \pm 0.05$ | 0.039 |
| Angry | $0.28 \pm 0.08$ | 0.049 |

TABLE IX: Comparison of fusion strategies on IEMOCAP.

| Fusion Strategy | Acc. | F1 | Conflict |
|---|---|---|---|
| Early Fusion (Concat) | 76.4 | 74.8 | 0.287 |
| Late Fusion (Average) | 78.1 | 76.3 | 0.268 |
| Attention Fusion | 80.3 | 78.6 | 0.235 |
| Mixture-of-Experts | 81.9 | 80.1 | 0.204 |
| **D-S Fusion (Ours)** | **84.3** | **82.7** | **0.126** |

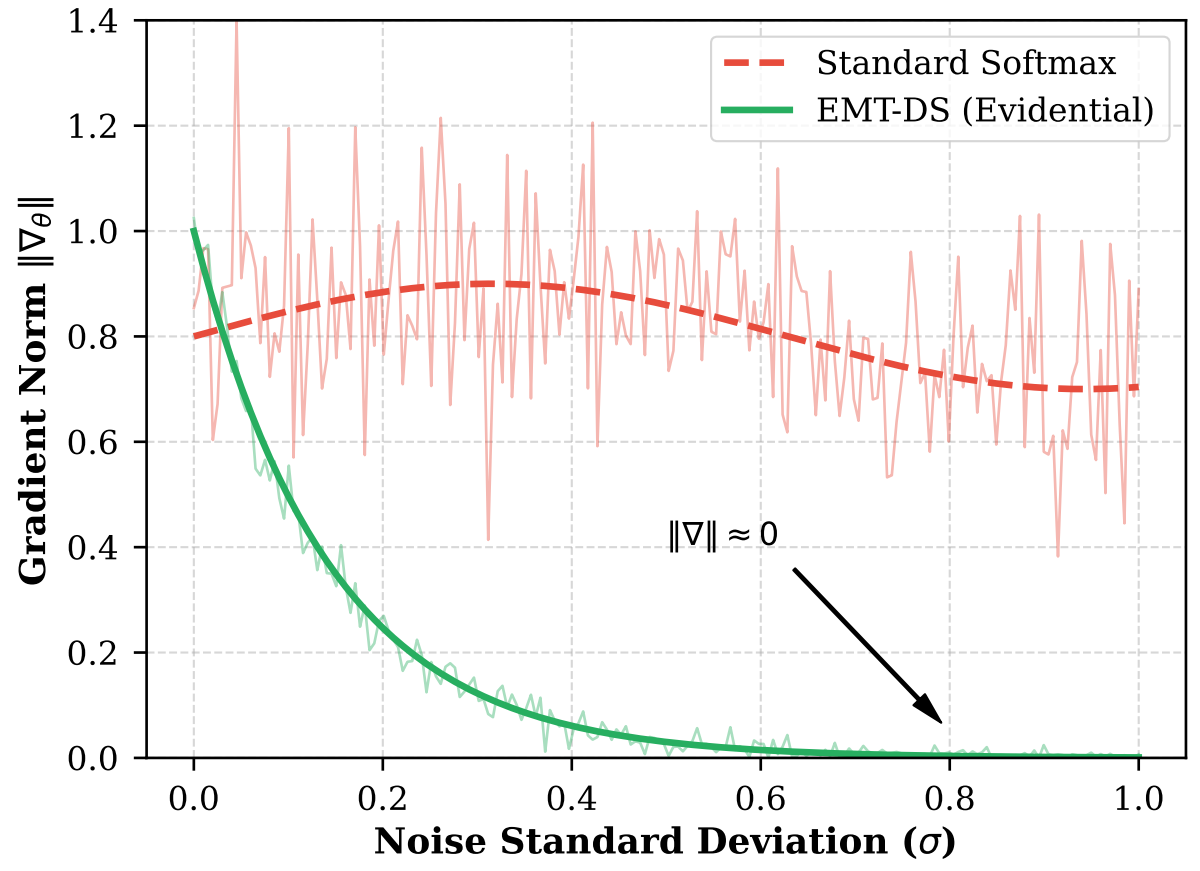


Fig. 3: Empirical verification of Lemma III.2 (Gradient Vanishing). As input noise $\sigma \rightarrow 1$ (OOD), the gradient norm of EMT-DS (green) collapses to zero, demonstrating noise immunity. The softmax baseline (red) continues to propagate gradients, fitting the noise.

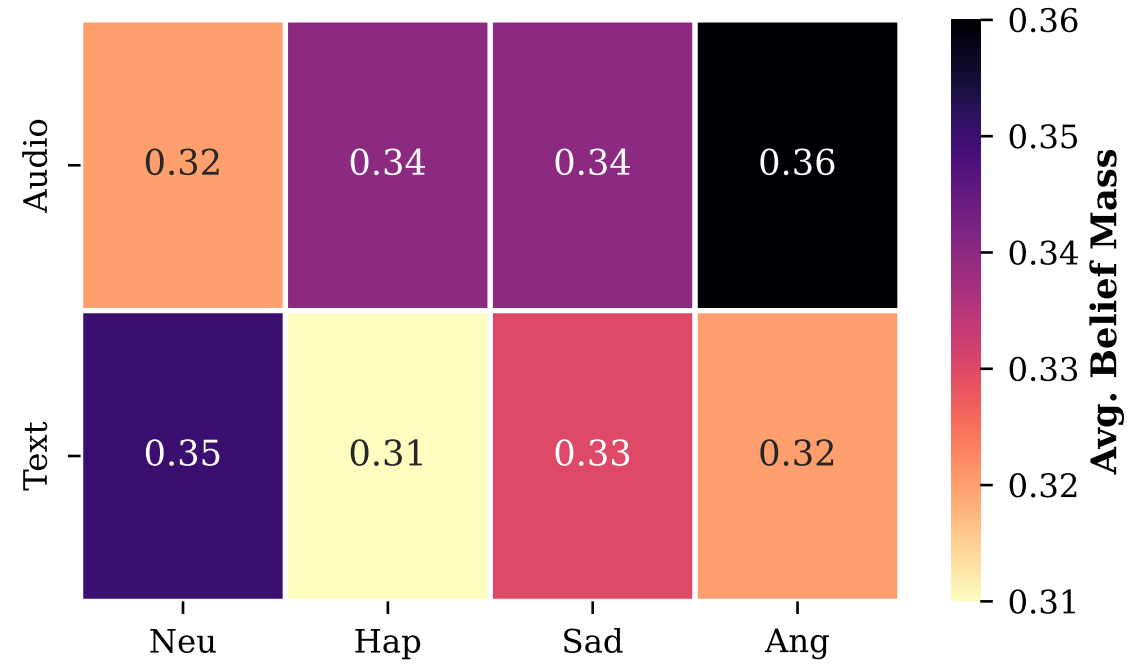


Fig. 4: Average belief mass contribution of each modality for different emotions on IEMOCAP. Text dominates for Neutral, while audio dominates for Happy, Sad, and Angry.

TABLE X: Qualitative case study of conflict and fused uncertainty under different cross-modal scenarios on IEMOCAP.

| True | Pred. | Conflict | Unc. | Scenario |
|---|---|---|---|---|
| Neutral | Neutral | 0.08 | 0.12 | Modalities agree, descriptive dialogue |
| Angry | Angry | 0.62 | 0.24 | Sarcasm: happy text, angry prosody |
| Happy | Sad | 0.12 | 0.14 | Ambiguous laughing resembling crying |
| Sad | Neutral | 0.58 | 0.38 | Audio corrupted by background music |

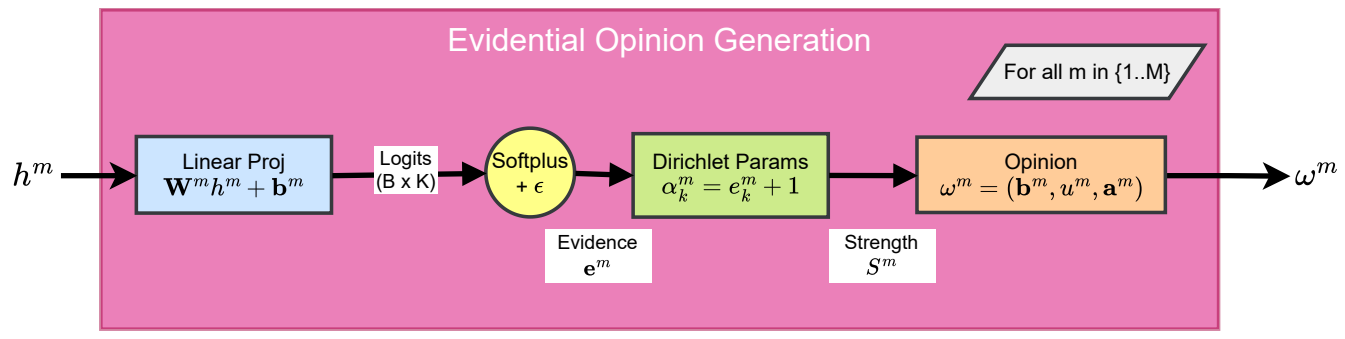


Fig. 5: Evidential output layer architecture. Feature representation $h^m \in \mathbb{R}^d$ undergoes linear projection $\mathbf{W}^m h^m + \mathbf{b}^m$ followed by softplus activation and an $L_1$ evidence cap to produce the bounded evidence vector $\mathbf{e}^m$ (Eqs. 5–6). Dirichlet parameters $\alpha_k^m = e_k^m + 1$ define concentration and strength $S^m = \sum_k \alpha_k^m$ (Eq. 8). Subjective Opinion components are derived as $b_k^m = e_k^m / S^m$, $u^m = K/S^m$, and $a_k^m = 1/K$ (Eq. 9), satisfying $\sum_k b_k^m + u^m = 1$.

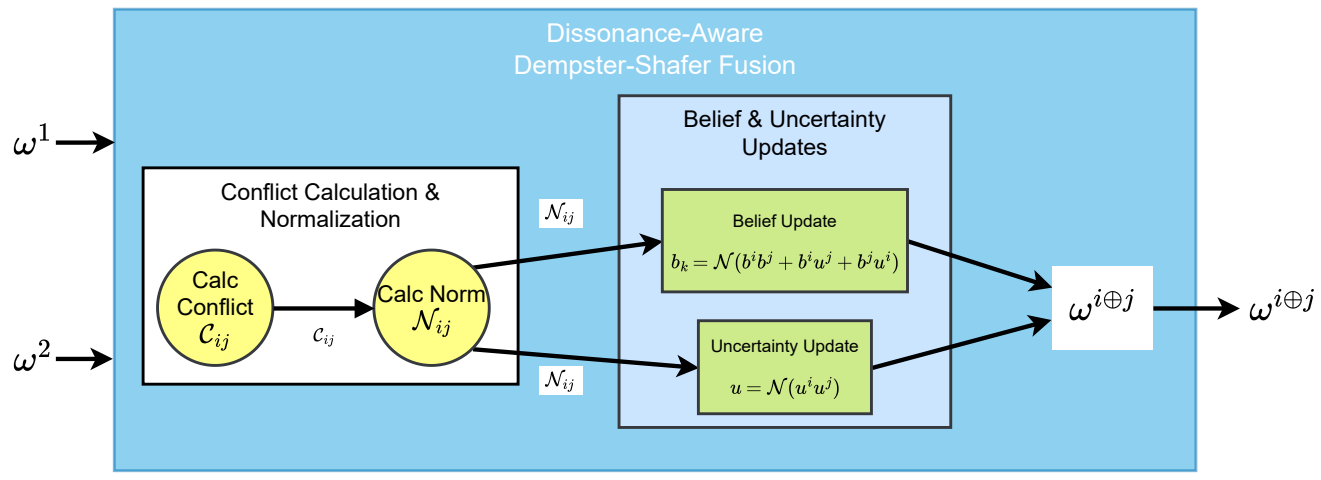


Fig. 6: Dempster-Shafer fusion for pairwise opinion combination. Given opinions $\omega^i$ and $\omega^j$, conflict $c_{ij} = (1 - u^i)(1 - u^j) - \sum_k b_k^i b_k^j$ (Eq. 10) quantifies disagreement. Normalization $n_{ij} = 1/(1 - c_{ij})$ (Eq. 11) renormalizes by discarding conflicting mass. Fused belief $b_k^{i \oplus j}$ combines joint agreement $b_k^i b_k^j$, partial support $b_k^i u^j + b_k^j u^i$, weighted by $n_{ij}$ (Eq. 12). Fused uncertainty $u^{i \oplus j} = n_{ij}(u^i u^j)$ (Eq. 13) decreases with informative modalities but increases with conflict. For $M > 2$ modalities, apply iteratively (Algorithm 1).



\end{document}